\documentclass{article}

\usepackage{microtype}
\usepackage{graphicx}
\usepackage{booktabs} 

\usepackage{hyperref}

\usepackage[accepted]{icml2023}

\usepackage{amsmath}
\usepackage{amssymb}
\usepackage{mathtools}
\usepackage{amsthm}

\usepackage[capitalize,noabbrev]{cleveref}

\theoremstyle{plain}
\newtheorem{theorem}{Theorem}[section]

\newtheorem{lemma}[theorem]{Lemma}

\theoremstyle{definition}
\newtheorem{definition}[theorem]{Definition}
\newtheorem{assumption}[theorem]{Assumption}
\theoremstyle{remark}

\usepackage[textsize=tiny]{todonotes}

\usepackage{booktabs}
\usepackage{multirow}
\usepackage{graphicx}
\usepackage{subcaption}

\newlength{\twosubht}
\newsavebox{\twosubbox}

\usepackage{amsmath,amsfonts,bm}

\def\eqref#1{equation~\ref{#1}}

\def\1{\bm{1}}
\newcommand{\train}{\mathcal{D}}

\def\rr{{\textnormal{r}}}

\def\vtheta{{\bm{\theta}}}

\def\vb{{\bm{b}}}
\def\vc{{\bm{c}}}

\def\vg{{\bm{g}}}

\def\vx{{\bm{x}}}
\def\vy{{\bm{y}}}

\DeclareMathAlphabet{\mathsfit}{\encodingdefault}{\sfdefault}{m}{sl}
\SetMathAlphabet{\mathsfit}{bold}{\encodingdefault}{\sfdefault}{bx}{n}

\def\gN{{\mathcal{N}}}

\newcommand{\E}{\mathbb{E}}

\newcommand{\R}{\mathbb{R}}

\DeclareMathOperator*{\argmax}{arg\,max}

\newboolean{showcomments}
\setboolean{showcomments}{true} 
\ifthenelse{\boolean{showcomments}}
  {\newcommand{\nb}[2]{
    \fcolorbox{gray}{yellow}{\bfseries\sffamily\scriptsize#1}
    {$\blacktriangleright$#2$\blacktriangleleft$}
   }
   
  }
  {\newcommand{\nb}[2]{}
   
  }
\newcommand \hta[1]{\nb{HTA}{\textcolor{blue}{\textsl{#1}}}} 

\newcommand{\probdist}{\mathcal{D}} 
\newcommand{\probinst}{P} 
\newcommand{\feas}{\Omega} 
\newcommand{\origobj}{f} 
\newcommand{\features}{\phi} 
\newcommand{\surrobjlin}{\hat{c}} 
\newcommand{\nn}{NN} 
\newcommand{\nnparams}{\bm{\theta}} 

\newcommand{\loss}{\mathcal{L}} 

\newcommand{\ours}{\texttt{SurCo}}

\newcommand{\optver}{\texttt{zero}}
\newcommand{\optvershort}{\texttt{zero}}
\newcommand{\trainver}{\texttt{prior}}
\newcommand{\trainvershort}{\texttt{prior}}

\newcommand{\hybridver}{\texttt{hybrid}}
\newcommand{\hybridvershort}{\texttt{hybrid}}

\newcommand{\baselinedomain}{\texttt{Domain-Heuristic}}
\newcommand{\baselinelookahead}{\texttt{Greedy}}
\newcommand{\baselinemartin}{\texttt{Pass-Through}}

\newcommand{\cD}{\mathcal{D}}

\def\zz{\mathbb{Z}}
\def\rr{\mathbb{R}}

\def\vtheta{\boldsymbol{\theta}}
\def\vphi{\boldsymbol{\phi}}
\def\train{\mathrm{train}}
\def\eval{\mathrm{eval}}
\def\direct{\mathrm{direct}}
\def\cL{\mathcal{L}}

\def\vol{\mathrm{vol}}

\newcommand\revision[1]{{\color{blue}#1}}

\usepackage{hyperref}
\usepackage{url}

\usepackage{thmtools}
\usepackage{thm-restate}

\icmltitlerunning{\ours{}: Learning Linear SURrogates for COmbinatorial Nonlinear Optimization Problems}

\begin{document}

\twocolumn[
\icmltitle{\ours{}: Learning Linear SURrogates \\ for COmbinatorial Nonlinear Optimization Problems}



\icmlsetsymbol{equal}{*}

\begin{icmlauthorlist}
\icmlauthor{Aaron Ferber}{usc}
\icmlauthor{Taoan Huang}{usc}
\icmlauthor{Daochen Zha}{rice}
\\
\icmlauthor{Martin Schubert}{fair}
\icmlauthor{Benoit Steiner}{ant}
\icmlauthor{Bistra Dilkina}{usc}
\icmlauthor{Yuandong Tian}{fair}

\end{icmlauthorlist}

\icmlaffiliation{fair}{Meta AI, FAIR}
\icmlaffiliation{usc}{Center for AI in Society, University of Southern California}
\icmlaffiliation{rice}{Rice University}
\icmlaffiliation{ant}{Anthropic}


\icmlcorrespondingauthor{Aaron Ferber}{aferber@usc.edu}
\icmlcorrespondingauthor{Yuandong Tian}{yuandong@meta.com}

\icmlkeywords{Machine Learning, Combinatorial Optimization, Nonlinear Optimization}

\vskip 0.3in
]



\printAffiliationsAndNotice{Work done during Aaron and Taoan's internship in Meta AI. Project page at \url{https://sites.google.com/usc.edu/surco/}}  

\def\surco{\texttt{SurCo}}
\def\szero{\texttt{SurCo}-\texttt{zero}}
\def\sprior{\texttt{SurCo}-\texttt{prior}}
\def\shybrid{\texttt{SurCo}-\texttt{hybrid}}

\begin{abstract}
Optimization problems with nonlinear cost functions and combinatorial constraints appear in many real-world applications but remain challenging to solve efficiently compared to their linear counterparts. To bridge this gap, we propose $\textbf{\emph{\texttt{SurCo}}}$ that learns linear $\underline{\text{Sur}}$rogate costs which can be used in existing $\underline{\text{Co}}$mbinatorial solvers to output good solutions to the original nonlinear combinatorial optimization problem. The surrogate costs are learned end-to-end with nonlinear loss by differentiating through the linear surrogate solver, combining the flexibility of gradient-based methods with the structure of linear combinatorial optimization. We propose three $\texttt{SurCo}$ variants: $\texttt{SurCo}-\texttt{zero}$ for individual nonlinear problems, $\texttt{SurCo}-\texttt{prior}$ for problem distributions, and $\texttt{SurCo}-\texttt{hybrid}$ to combine both distribution and problem-specific information. We give theoretical intuition motivating $\texttt{SurCo}$, and evaluate it empirically. Experiments show that $\texttt{SurCo}$ finds better solutions faster than state-of-the-art and domain expert approaches in real-world optimization problems such as embedding table sharding, inverse photonic design, and nonlinear route planning. 

\end{abstract}

\begin{figure*}
    \centering
    \includegraphics[width=0.75\textwidth]{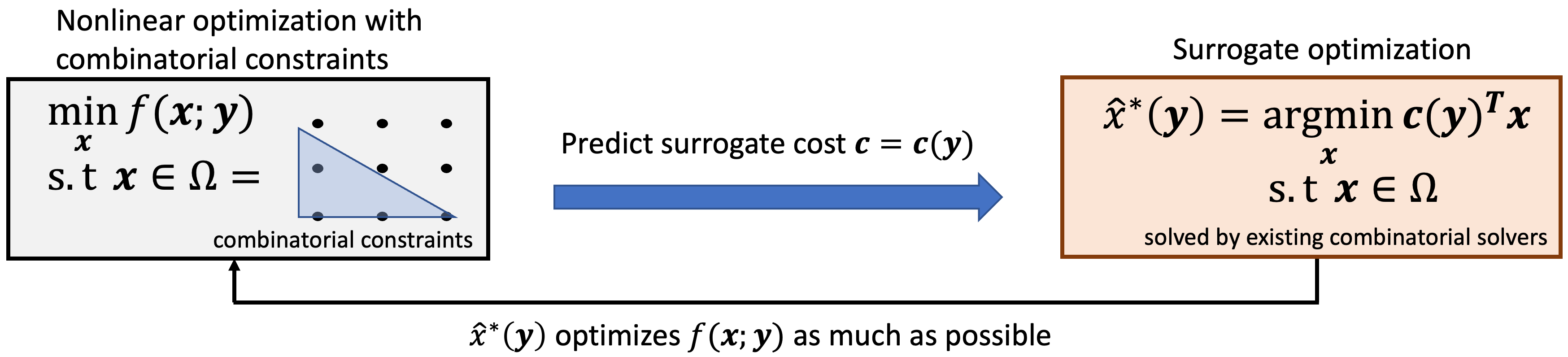}
    \caption{\small Overview of our proposed framework \ours{}.
    }
    \label{fig:overview}
\end{figure*}

\section{Introduction}

Combinatorial optimization problems with linear objective functions such as  mixed integer linear programming (MILP)~\citep{wolsey2007mixed}, and occasionally linear programming (LP)~\citep{chvatal1983linear}, have been extensively studied in operations research (OR). The resulting high-performance solvers like Gurobi~\citep{gurobi} can solve industrial-scale optimization problems with tens of thousands of variables in a few minutes. 

However, even with perfect solvers, one issue remains: the cost functions $f(\vx)$ in many practical problems are \emph{nonlinear}, and the highly-optimized solvers mainly handle linear or convex formulations
while real-world problems have less constrained objectives.
For example, in embedding table sharding~\citep{zha2022autoshard} one needs to distribute embedding tables to multiple GPUs for the deployment of recommendation systems. Due to the batching behaviors within a single GPU and communication cost among different GPUs, the overall latency (cost function) in this application depends on interactions of multiple tables and thus can be highly nonlinear~\citep{zha2022autoshard}. 

To obtain useful solutions to real-world problems, one may choose to directly optimize the nonlinear cost, which can be the black-box output of a simulator~\citep{gosavi2015simulation,ye2019automated}, or the output of a cost estimator learned by machine learning techniques (e.g., deep models) from offline data~\citep{steiner2021value,koziel2021accurate,wang2021surrogate,cozad2014learning}. However, many of these direct optimization approaches either rely on human-defined heuristics (e.g., greedy~\citep{korte1978analysis,reingold1981greedy,wolsey1982analysis}, local improvement~\citep{voss2012meta,li2021learning}), or resort to general nonlinear optimization techniques like gradient descent~\citep{ruder2016overview}, reinforcement learning~\citep{mazyavkina2021reinforcement}, or evolutionary algorithms~\citep{simon2013evolutionary}. While these approaches can work in certain settings, they may lead to a slow optimization process, in particular when the cost function is expensive to evaluate, and they often ignore the combinatorial nature of most real-world applications.

In this work, we propose a systematic framework \textbf{\emph{\ours{}}} that leverages existing efficient combinatorial solvers to find solutions to nonlinear combinatorial optimization problems arising in real-world scenarios. 
When only one nonlinear \emph{differentiable} cost $f(\vx)$ needs to be minimized, we propose \ours{}-\optver{} that optimizes a \emph{linear surrogate} cost $\hat\vc$ so that the \emph{surrogate optimizer} (SO) $\min_{\vx\in\Omega} \hat\vc^\top \vx$ outputs a solution that is expected to be optimal w.r.t. the \emph{original} nonlinear cost $f(\vx)$. Due to its linear nature, SO can be solved efficiently with existing solvers, and the surrogate cost $\hat\vc$ can be optimized in an end-to-end manner by back-propagating \emph{through} the solver via methods proposed in previous work~\citep{poganvcic2019diffbb,niepert2021imle,berthet2020diffperturb}. 

Thus, \ours{} is a general-purpose method for solving combinatorial nonlinear optimization. Off-the-shelf nonlinear optimizers are often not directly applicable to these problem domains and often require domain-specific solution methodologies to give high-quality solutions in a reasonable amount of time, and solution prediction methods fail to give combinatorially feasible solutions without problem-specific intervention. Here, learning a linear surrogate problem ensures that the surrogate solver is practically efficient, yields gradient information for offline training, and generates solutions that are combinatorially feasible.


When solving a family of nonlinear differentiable functions $f(\vx;\vy)$ parameterized by instance description $\vy$, the surrogate \emph{coefficients} $\hat\vc(\vy;\vtheta)$ are learned on a set of optimization instances (called the training set $\{\vy_i\}$), by optimizing the parameters $\vtheta$. For an unseen held-out instance $\vy'$, we propose \ours{}-\trainver{} that directly optimizes linear SO: $\hat{\vx}^*(\vy') := \arg\min_{\vx\in\Omega(\vy')} \hat\vc^\top(\vy';\vtheta) \vx$ to get the solution,  avoiding optimizing the cost $f(\vx;\vy')$ from scratch. Based on the solution predicted by \ours{}-\trainver{}, we also propose \ours{}-\hybridver{} that fine-tunes the surrogate costs $\hat\vc$ with \ours{}-\optver{} to leverage both domain knowledge synthesized offline and information about the specific instance. We provide a comprehensive description of \ours{} in Section 3.


We evaluate \ours{} in three settings: embedding table sharding \citep{zha2022autoshard}, photonic inverse design \citep{schubert_inverse_2022}, and nonlinear route planning \cite{fan2005arrivingontime}. In the on-the-fly setting, \ours{}-\optvershort{} achieves higher quality solutions in comparable or less runtime, thanks to the help of an efficient combinatorial solver. in \ours{}-\trainver{}, our method obtains better solutions in held-out problems compared to other methods that require training (e.g., reinforcement learning).

We compare \ours{} at a high level with related work integrating learning and optimization at the end of our paper. We additionally present theoretical intuition that helps motivate why training a model to predict surrogate linear coefficients may exhibit better sample complexity than previous approaches that directly predict the optimal solution~\citep{li2018combinatorial,ban2019big}.







\begin{table}
\small
\begin{tabular}{c|l}
Symbol & Description \\
\hline\hline
$\vy$ & Parametric description of a specific instance. \\
$\vx$ & A solution to an instance. \\
$f(\vx;\vy)$ & The nonlinear objective (w.r.t $\vx$) for an instance $\vy$. \\
$\Omega(\vy)$ & The feasible region of an instance $\vy$. \\
$\hat \vx^*(\vy)$ & The optimal SO solution to an instance $\vy$. \\
$\vc(\vy)$ & The surrogate coefficients for instance $\vy$.  
\end{tabular}
\caption{\small Notations used in this work.}
\end{table}

\section{Problem Specification}
Our goal is to solve the following nonlinear optimization problem describe by $\vy$: 
\begin{equation}\small
\min_\vx f(\vx; \vy) \quad\quad \mathrm{s.t.}\quad \vx \in \feas{(\vy)}
\end{equation}
where $\vx\in\rr^n$ are the $n$ variables to be optimized, $f(\vx;\vy)$ is the nonlinear differentiable cost function to be minimized, $\feas{(\vy)}$ is the feasible region, typically specified by linear (in)equalities and integer constraints, and $\vy\in Y$ are the problem instance parameters drawn from a distribution $\probdist{}$ over $Y$. For example, in the traveling salesman problem, $\vy$ can be the distance matrix among cities. 

\textbf{Differentiable cost function}. The nonlinear cost function $f(\vx;\vy)$ can either be given analytically, or the result of a simulator made differentiable via finite differencing (e.g., JAX \citep{jax2018github}). If the cost function $f(\vx;\vy)$ is not differentiable as in one of our experimental settings, we can use a cost model that is learned from an offline dataset, often generated via sampling multiple feasible solutions within $\Omega(\vy)$, and recording their costs. In this work, we assume the following property of $f(\vx;\vy)$:
\begin{assumption}[Differentiable cost function]
During optimization, the cost function $f(\vx;\vy)$ and its partial derivative $\partial f/\partial \vx$ are accessible. 
\end{assumption}
Learning a good nonlinear cost model $f$ is non-trivial for practical applications (e.g., AlphaFold~\citep{jumper2021highly}, Density Functional Theory~\citep{nagai2020completing}, cost model for embedding tables~\citep{zha2022autoshard}) and is beyond the scope of this work.  

\textbf{Evaluation Metric}. We mainly focus on two aspects: the solution quality evaluated by $f(\hat \vx; \vy)$, and the number of queries of $f$ during optimization to achieve the solution $\hat\vx$. For both, smaller measurements are favorable, i.e., fewer query of $f$ to get solutions closer to global optimum. 

When $f(\vx; \vy)$ is linear w.r.t $\vx$, and the feasible region $\Omega(\vy)$ can be encoded using mixed integer programs or other
mathematical programs, the problem can be solved efficiently using existing scalable optimization solvers. When $f(\vx;\vy)$ is nonlinear, we propose \ours{} that learns a surrogate
linear objective function, which allow us to leverage these existing scalable optimization solvers, and which results in a solution that has high quality with respect to the original hard-to-encode objective function $f(\vx;\vy)$. 

\begin{table*}
\scriptsize
\centering
\begin{tabular}{c|ccccc}
\multirow{2}{*}{Methods} & Applicable to & Objective can be & \multirow{2}{*}{Training Set} & Generalize to & Built-in handling of  \\
& nonlinear objective & free form & & unseen instances & combinatorial constraints\\ 
\hline
\hline
Gradient Descent & Yes & Yes & N/A & No & No \\
Evolutionary Algorithm & Yes & Yes & N/A & No & No \\
Nonlinear combinatorial solvers & Yes & No & N/A & No & Yes \\ 
Learning direct mapping & Yes & Yes & $\{\vy_i, \vx^*_i\}$ & Yes & No \\
Predict-then-optimize   & Limited & No & $\{\vy_i, \vx^*_i\}$ & Yes & Yes \\
\hline
SurCo (proposed)   & Yes & Yes & $\{\vy_i\}$ & Yes & Yes 
\end{tabular}
\caption{\small Conceptual comparison of optimizers (both traditional and ML-guided). Our approach (SurCo) can handle nonlinear objective without a predefined analytical form, does not require pre-computed optimal solutions in its training set, can handle combinatorial constraints (via commercial solvers it incorporates), and can generalize to unseen instances.} 
\end{table*}

\section{\ours: Learning Linear Surrogates}
\subsection{\ours{}-\optver{}: on-the-fly optimization}

We start from the simplest case in which we focus on a single  instance with $f(\vx) = f(\vx;\vy)$ and $\Omega = \Omega(\vy)$. \ours{}-\optvershort{} aims to optimize the following objective:
\begin{equation}\small
\text{(\ours{}-\optvershort)}: \quad \min_{\vc} \cL_\optvershort(\vc) := f(\vg_\Omega(\vc)) \label{eq:ours-optver}
\end{equation}
where the surrogate optimizer $\vg_\Omega: \rr^n \mapsto \rr^n$ is the output of certain combinatorial solvers with linear cost weight $\vc\in \rr^n$ and feasible region $\Omega \subseteq \rr^n$. For example, $\vg_\Omega$ can be the following: 
\begin{equation}
\small
    \vg_\Omega(\vc) := \arg\min_\vx \vc^\top \vx \quad \mathrm{s.t.}\ \ \vx\in \Omega := \{A\vx \le \vb, \vx \in \zz^n\} 
\end{equation}
which is the output of a MILP solver. Thanks to previous works~\citep{ferber2020mipaal, poganvcic2019diffbb}, we can efficiently compute the partial derivative $\partial \vg_\Omega(\vc) / \partial \vc$. Intuitively, this means that $\vg_\Omega(\vc)$ can be \emph{backpropagated} through. 
Since $f$ is also differentiable with respect to the solution it is evaluating, we thus can optimize Eqn.~\ref{eq:ours-optver} in an end-to-end manner using any gradient-based optimizer:
\begin{equation}\small
\vc(t+1) = \vc(t) - \alpha \frac{\partial \vg_\Omega}{\partial \vc} \frac{\partial f}{\partial \vx}, 
\end{equation}
where $\alpha$ is the learning rate. The procedure starts from a randomly initialized $\vc(0)$ and converges at a local optimal solution of $\vc$. 
While Eqn.~\ref{eq:ours-optver} is still nonlinear optimization and there is no guarantee about the quality of the final solution $\vc$, we argue that optimizing Eqn.~\ref{eq:ours-optver} is better than optimizing the original nonlinear cost $\min_{\vx\in\Omega} f(\vx)$. Furthermore, while we cannot guarantee optimality, we  guarantee feasibility by leveraging a linear combinatorial solver. 

Intuitively, instead of optimizing directly over the solution space $\vx$, we optimize over the space of surrogate costs $\vc$, and delegate the combinatorial feasibility requirements of the nonlinear problem to SoTA combinatorial solvers. Compared to naive approaches that directly optimize $f(\vx)$ via general optimization techniques, our method readily handles complex constraints of the feasible regions, and thus makes the optimization procedure easier. Furthermore, it also helps escape from local minima, thanks to the embedded search component of existing combinatorial solvers (e.g., branch-and-bound~\citep{land2010automatic} in MILP solvers). As we see in the experiments, this is particularly important when the problem becomes large-scale with more local optima. This approach works well when we are optimizing individual instances and may not have access to offline training data or the training time is cost-prohibitive. 

\textbf{Limitation.} Note that due to linear surrogate, our approach will always return a vertex in the feasible region, while the solution to the original nonlinear objective may be in the interior. We leave this limitation for future work. In many real-world settings, such as in the three domains we tested, the solutions are indeed on the vertices of feasible regions.  




\subsection{\ours{}-\trainver{}: offline surrogate training}

We now consider a more general case where we have $N$ optimization instances, each parameterized by an instance description $\vy_i$, $i = 1\ldots N$, and we want to find their solutions to a \emph{collection} of nonlinear loss functions $f(\vx;\vy_i)$ simultaneously. Here we write $\cD_\train := \{\vy_i\}_{i=1}^N$ as the training set.  
A naive approach is just to apply \ours{}-\optver{} $N$ times, which leads to $N$ independent surrogate costs $\{\vc_i\}_{i=1}^N$. However, this approach does not consider two important characteristics. 
\emph{First}, it fails to leverage possible relationship between the instance descriptor $\vy_i$ and its associated surrogate cost $\vc_i$, since every surrogate cost is independently estimated. \emph{Second}, it fails to learn any useful knowledge from the $N$ instances after optimization. As a result, for an unseen instance, the entire optimization process needs to be conducted again, which is slow. 
This motivates us to add a surrogate cost \emph{model} $\hat\vc(\vy;\vtheta)$ into the optimization as a regularizer: 
\begin{eqnarray}\small
    \text{(\ours{}-\trainvershort-$\lambda$)}: \quad \min_{\vtheta, \{\vc_i\}} \cL_\trainvershort(\vtheta,\{\vc_i\};\lambda) \nonumber \\
    := \sum_{i=1}^N f(\vg_{\Omega(\vy_i)}(\vc_i);\vy_i) + \lambda \|\vc_i - \hat\vc(\vy_i;\vtheta))\|_2 
    \label{eq:ours-trainver-lambda}
\end{eqnarray}
The regressor model $\hat\vc(\vy;\vtheta)$ directly predicts the surrogate cost from the instance description. The form of the regressor can be a neural network, in which $\vtheta$ is its parameters. Note that when $\lambda = 0$, it reduces to $N$ independent optimizations, while when $\lambda > 0$, the surrogate costs $\{\vc_i\}$ interact with each other. 
With the regressor, we distill knowledge gained from the optimization procedure into $\vtheta$, which can be used for an unseen instance $\vy'$. Indeed, we use the learned regressor model to predict the surrogate cost $\vc' = \hat\vc(\vy';\vtheta)$, and directly solve the \emph{surrogate optimization} (SO): 
\begin{equation}\small
\hat \vx^*(\vy') = \arg\min_{\vx\in\Omega(\vy)} \hat\vc^\top(\vy';\vtheta) \vx
\end{equation}
A special case is when $\lambda\rightarrow +\infty$, we directly learn the network parameters $\vtheta$ instead of individual surrogate costs:
\begin{eqnarray}\small
    \text{(\ours{}-\trainvershort)}: \quad \min_{\vtheta} \cL_\trainvershort(\vtheta) \nonumber \\
    := \sum_{i=1}^N f(\vg_{\Omega(\vy_i)}(\hat\vc(\vy_i;\vtheta));\vy_i) \label{eq:ours-trainver}
\end{eqnarray}
This approach is useful when the goal is to find high-quality solutions for unseen instances of a problem distribution when the upfront cost of offline training is acceptable but the cost of optimizing on-the-fly is prohibitive. Here, we require access to a distribution of training optimization problems, but at test time only require the feasible region and not the nonlinear objective. Different from predict-then-optimize~\cite{elmachtoub2022smart,ferber2020mipaal} or ML optimizers~\cite{ban2019big}, we do not require the optimal solution $\{\vx^*_i\}_{i=1}^N$ as part of the training set.


\subsection{\ours{}-\hybridvershort{}: fine-tuning a predicted surrogate}
Naturally, we consider \ours{}-\hybridvershort{}, a hybrid approach which initializes the coefficients of \ours{}-\optvershort{} with the coefficients predicted from \ours{}-\trainvershort{} which was trained on offline data. This allows \ours{}-\hybridvershort{} to start out optimization from an initial prediction that has good performance for the distribution at large but which is then fine-tuned for the specific instance. Formally, we initialize $\vc(0) = \hat\vc(\vy_i;\vtheta)$ and then continue optimizing $\vc$ based on the update from \ours{}-\optvershort{}. This approach is geared towards optimizing the nonlinear objective using a high-quality initial prediction that is based on the problem distribution and then fine-tuning the objective coefficients based on the specific problem instance at test time. Here, high performance comes at the runtime cost of both having to train offline on a problem distribution as well as performing fine-tuning steps on-the-fly. However, this additional cost is often worthwhile when the main goal is to find the best possible solutions by leveraging synthesized domain knowledge in combination with individual problem instances as arises in chip design \citep{mirhoseini2021chipdesign} and compiler optimization \citep{zhou2020compiler}. 

\section{Is Predicting Surrogate Cost better than Predicting Solution? A Theoretical Analysis}

One of the key ingredient of our proposed methods (\ours{}-\trainvershort{} and \ours{}-\hybridvershort{}) is to learn a model to predict surrogate cost $\vc$ from instance description $\vy$, which is in contrast with previous solution regression approaches that directly learn a mapping from problem description $\vy$ to the solution $\vx^*(\vy)$ ~\citep{ban2019big}. A natural question arise: which one is better? 

In this section, we give theoretical intuition to compare the two approaches using a simple 1-nearest-neighbor (1-NN) solution regressor \citep{fix19851nn}. We first relate the number of samples needed to learn any mapping to its \emph{Lipschitz constant} $L$, and then show that for the direct mapping $\vy\mapsto \vx^*(\vy)$, $L$ can be very large. Therefore, there exist fundamental difficulties to learn such a mapping. When this happens, we can still find surrogate cost mapping $\vy\mapsto \vc^*(\vy)$ with finite $L$ that leads to the optimal solution $\vx^*(\vy)$ of the original nonlinear problems. 

\subsection{Lipschitz constant and sample complexity}
\label{sec:lipschitz-sample-complexity}

\def\cC{\mathcal{C}}

Formally, consider fitting any mapping $\vphi: \rr^d \supseteq Y \mapsto \rr^m$ with a dataset $\cC := \{\vy_i, \vphi_i\}$. Here $Y$ is a compact region with finite volume $\vol(Y)$.  The Lipschitz constant $L$ is the smallest number so that $\|\vphi(\vy_1) - \vphi(\vy_2)\|_2 \le L\|\vy_1-\vy_2\|_2$ holds for any $\vy_1,\vy_2\in Y$. The following theorem shows that if the dataset covers the space $Y$, we could achieve high accuracy prediction: $\|\vphi(\vy) - \hat\vphi(\vy)\|_2 \le \epsilon$ for any $\vy\in Y$.

\begin{definition}[$\delta$-cover]
A dataset $\cC:=\{(\vy_i, \vphi_i)\}_{i=1}^N$ $\delta$-covers the space $Y$, if for any $\vy\in Y$, there exists at least one $\vy_i$ so that $\|\vy-\vy_i\|_2 \le \delta$. 
\end{definition}

\def\nn{\mathrm{nn}}

\begin{restatable}[Sufficient condition of prediction with $\epsilon$-accuracy]{lemma}{suffcondition}
\label{lemma::sample_complexity}
If the dataset $\cC$ can $(\epsilon/L)$-cover $Y$, then for any $\vy \in Y$, a 1-nearest-neighbor regressor $\hat\vphi$ leads to $\|\hat\vphi(\vy) - \vphi(\vy)\|_2 \le \epsilon$.
\end{restatable}

\begin{restatable}[Lower bound of sample complexity for $\epsilon/L$-cover]{lemma}{lowerboundcomplexity}
To achieve $\epsilon/L$-cover of $Y$, the size of the dataset set $N \ge N_0(\epsilon) := \frac{\vol(Y)}{\vol_0}\left(\frac{L}{\epsilon}\right)^d$, where $\vol_0$ is the volume of unit ball in $d$-dimension.
\end{restatable}
Please find all proofs in the Appendix. While we do not rule out a more advanced regressor than 1-nearest-neighbor that could lead to better sample complexity, the lemmas demonstrate that the Lipschitz constant $L$ plays an important role in sample complexity.

\begin{figure*}[ht!]
\sbox\twosubbox{%
    \includegraphics[height=4.5cm]{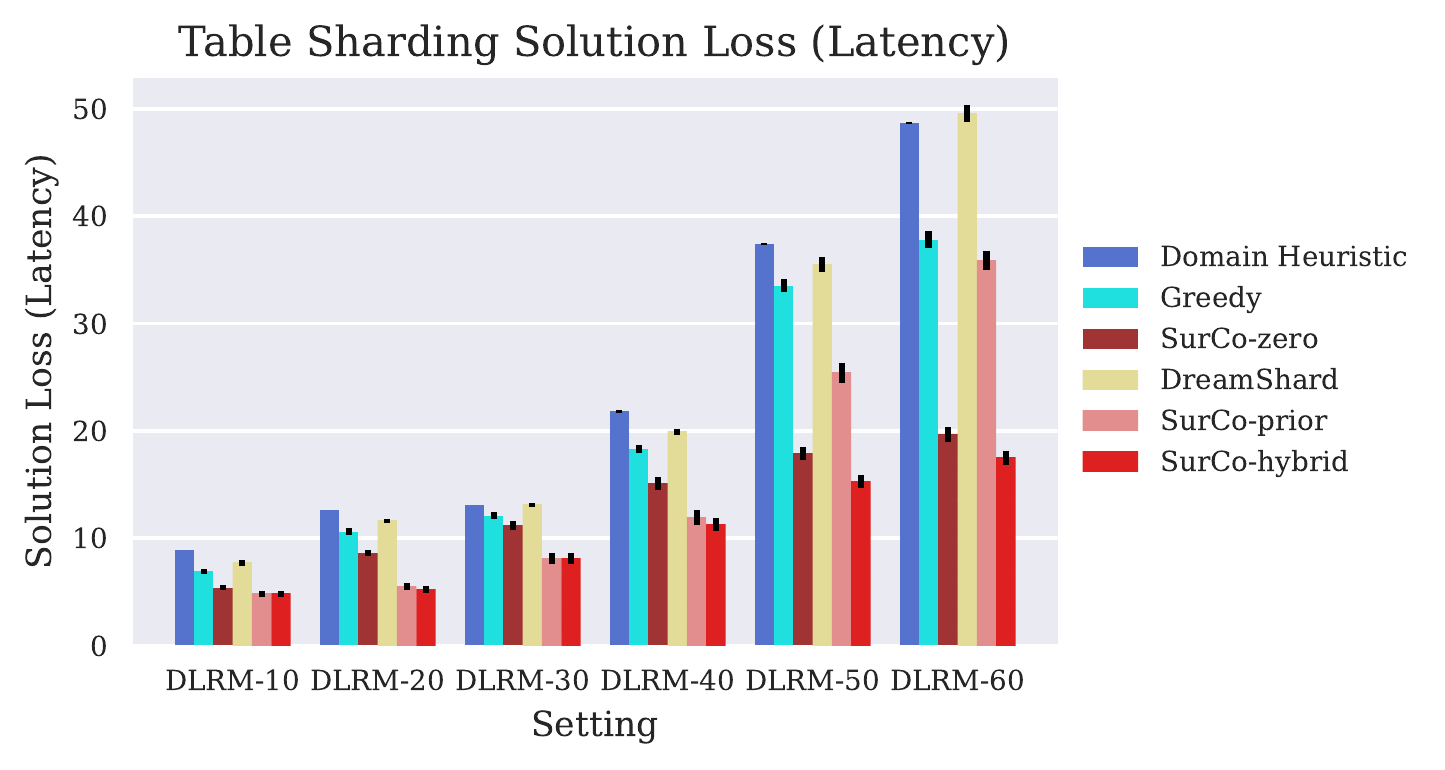}%
    \includegraphics[height=4.5cm]{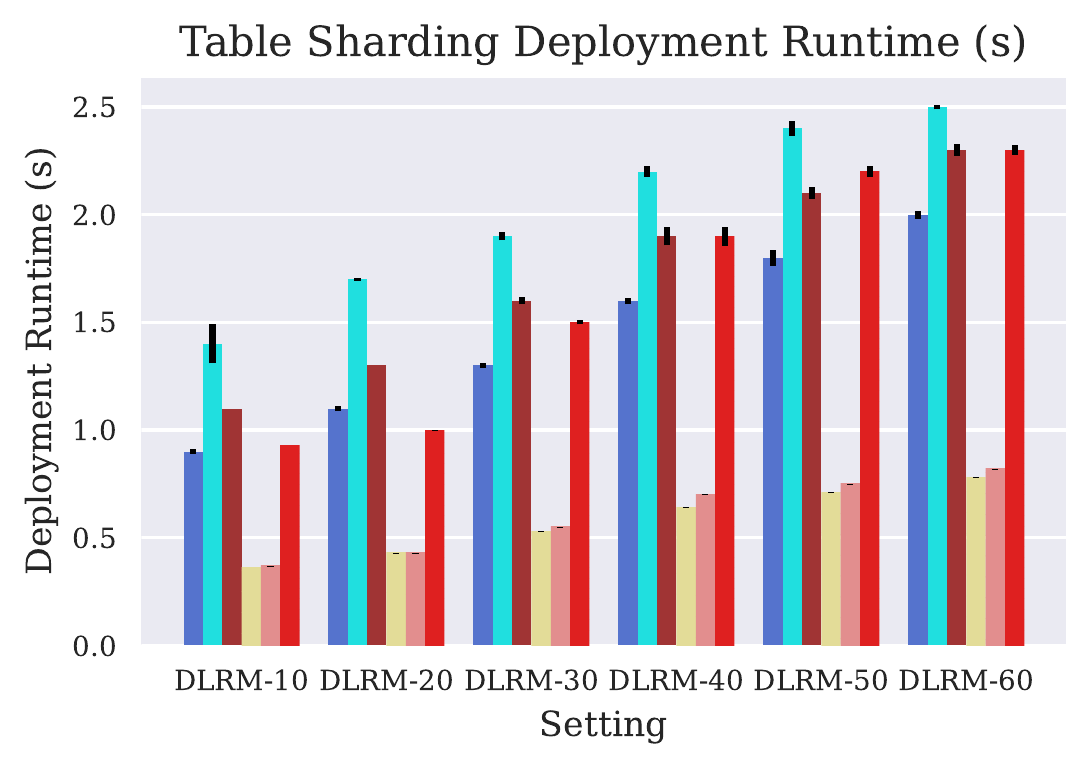}%
}
\setlength{\twosubht}{\ht\twosubbox}
\centering
\includegraphics[height=\twosubht]{plots/table_placement_quality.pdf}%
\includegraphics[height=\twosubht]{plots/table_placement_runtime.pdf}%

\caption{\small Table placement plan latency \small (\textbf{left}) and solver runtime \small (\textbf{right}). We evaluate \ours{} against Dreamshard \citep{zha2022dreamshard}, a SoTA offline RL sharding tool, a domain-heuristic of assigning tables based on dimension, and a greedy heuristic based on the estimated runtime increase. Striped approaches require pre-training.}
\label{fig:table_results}
\end{figure*}

\subsection{Difference between Cost and Solution Regression}
In the following we will show that in certain cases, the direct prediction $\vy \mapsto \vx^*(\vy)$ could have an infinitely large Lipschitz constant $L$.  
To show this, let us consider a general mapping $\vphi: \rr^d \supseteq Y \mapsto \rr^m$. Let $\vphi(Y)$ be the image of $Y$ under mapping $\vphi$ and $\kappa(Y)$ be the number of connected components for region $Y$. 
\begin{restatable}[A case of infinite Lipschitz constant]{theorem}{infinitelip}
\label{thm:cc-argument}
If the minimal distance $d_{\min}$ for different connected components of $\vphi(Y)$ is strictly positive, and $\kappa(\vphi(Y)) > \kappa(Y)$, then the Lipschitz constant of the mapping $\vphi$ is infinite. 
\end{restatable}
Note that this theorem applies to a wide variety of combinatorial optimization problems.
For example, when $Y$ is a connected region and the optimization problem can be formulated as an integer programming, the optimal solution set $\vx^*(Y) := \{\vx^*(\vy): \vy\in Y\}$ is a discrete set of integral vertices, 
so the theorem applies. 
Combined with analysis in Sec.~\ref{sec:lipschitz-sample-complexity}, we know the mapping $\vy \mapsto \vx^*(\vy)$ is hard to learn even with a lot of samples.  

We can see this more clearly with a concrete example in 2D space. Let the 1D instance description $y \in [0, \pi/2]$, and the feasible region is a convex hull of 3 vertices $\{(0, 0), (0, 1), (1, 0)\}$. The nonlinear objective is simply $f(\vx;y) := (x_1 \cos(y) + x_2 \sin(y))^2$, in which $\vx = (x_1,x_2)$ is the 2D solution vector. The direct mapping $y\rightarrow \vx^*$ maps a continuous region of instance descriptions (i.e., $y\in [0,\pi/2]$) into 2 disjoint regions points ($\vx^*=(0,1)$ and $\vx^*=(1,0)$), and thus according to Theorem~\ref{thm:cc-argument}, its Lipschitz constant must be infinite. In contrast, there exists a surrogate cost mapping $\vc(y) = [\cos(y), \sin(y)]^\top$, and the mapping $y\rightarrow \vc$ has finite Lipschitz constant (actually $L\le 1$) and can be learned easily.

\section{Empirical Evaluation}

We evaluate the variants of \ours{} on three settings, embedding table sharding, inverse photonic design, and nonlinear route planning, with the first two being real-world industrial settings. Each setting consists of a family of problem instances with varying feasible region and nonlinear objective function. Additionally, both table sharding and inverse photonic design lack analytical formulations of the objective function which prevents them from being used by many off-the-shelf nonlinear solvers like SCIP \citep{achterberg2009scip}.

\begin{figure*}[ht!]
\sbox\twosubbox{%
    \includegraphics[height=4.5cm]{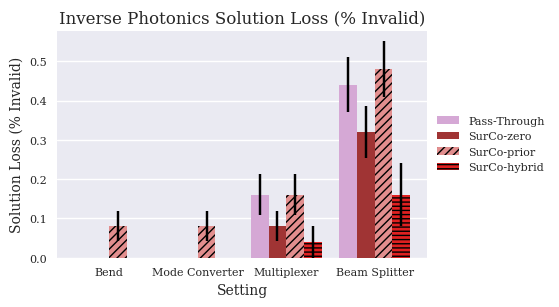}%
    \includegraphics[height=4.5cm]{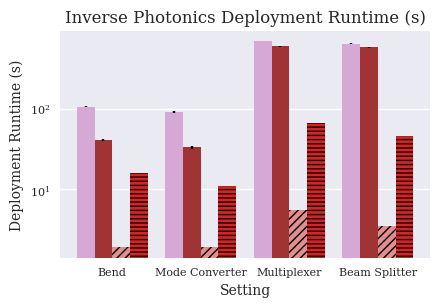}%
}
\setlength{\twosubht}{\ht\twosubbox}

\centering

\includegraphics[height=\twosubht]{plots/inverse_photonics_quality.png}%
\includegraphics[height=\twosubht]{plots/inverse_photonics_runtime.png}%

\caption{\small \textbf{Left} The solution loss (\% of failed instances when the design loss is not 0), and \textbf{right} test time solver runtime in log scale. For both, lower is better. We compare against the Pass-Through gradient approach proposed in \cite{schubert_inverse_2022}. We observe that \ours{}-\trainvershort{} achieves similar success rates to the previous approach \texttt{Pass-Through} with a substantially improved runtime. Additionally, \ours{}-\optvershort{} runs comparably or faster, while finding more valid solutions than \texttt{Pass-Through}. \ours{}-\hybridver{} obtains valid solutions most often and is faster than \ours{}-\optvershort{} at the expense of pretraining. Striped approaches use pretraining.}
\label{fig:inv_photo_results}
\end{figure*}
\subsection{Embedding Table Sharding}

The task of sharding embedding tables arises in the deployment of large-scale neural network models which operate over both sparse and dense inputs (e.g., in recommendation systems \citep{zha2022autoshard, zha2022dreamshard, zha2023pre, sethi2022recshard}).
Given $T$ embedding tables and $D$ homogeneous devices, the goal is to distribute the tables among the devices such that no device's memory limit is exceeded, while the tables are processed efficiently. Formally, let $x_{t,d}$ be the binary variable indicating whether table $t$ is assigned to device $d$, and $\vx := \{x_{t,d}\}\in \{0,1\}^{TD}$ be the collection of the variables. The optimization problem is $\min_{\vx\in\Omega} f(\vx;\vy)$ where $\Omega(\vy) := \left\{\vx: \forall t, \sum_t x_{t,d} = 1, \forall d, \sum_t m_t x_{t,d} \le M\right\}$. 

Here the problem description $\vy$ includes table memory usage $\{m_t\}$, and capacity $M$ of each device. $\sum_d x_{t,d} = 1$ means each table $t$ should be assigned to exactly one device, and $\sum_d m_t x_{t,d} \le M$ means the memory consumption at each device $d$ should not exceed its capacity. The nonlinear cost function $f(\vx;\vy)$ is the \emph{latency}, i.e., the runtime of the longest-running device. Due to shared computation (e.g., batching) among the group of assigned tables, and communication costs across devices, the objective is highly nonlinear. $f(\vx;\vy)$ is well-approximated by a sharding plan runtime estimator proposed by Dreamshard \citep{zha2022dreamshard}. Note that here, the runtime is approximated by a differentiable function since the real world deployment runtime isn't differentiable.

\ours{} learns to predict $T\times D$ surrogate cost $\surrobjlin{}_{t,d}$, one for each potential table-device assignment. During training, the gradients through the combinatorial solver $\partial\vg/\partial\vc$ are computed via CVXPYLayers \citep{agrawal2019cvxpylayers}, and the integrality constraints are relaxed. In practice, we obtained mostly integral solutions in that only one table on any given device was fractional. At test time, we solve for the integer solution using SCIP~\citep{achterberg2009scip}, a branch and bound MILP solver.   

\textbf{Settings.} We evaluate \ours{} on the public Deep Learning Recommendation Model (DLRM) dataset~\citep{DLRM19}. We consider 6 settings placing 10, 20, 30, 40, 50, and 60 tables on 4 devices, with a 5GB memory limit on GPU devices and 100 instances each (50 train, 50 test). 

\textbf{Baselines.} For impromptu baselines, \baselinelookahead{} allocates tables to devices based on predicted latency increase $f$, and the domain-expert algorithm \baselinedomain{} balances the aggregate dimension~\citep{zha2022dreamshard}. For \ours{}-\trainver{}, we use Dreamshard, the SoTA embedding table sharding algorithm using offline RL.

\textbf{Results.} Fig. \ref{fig:table_results}, \ours{}-\optvershort{} finds lower latency sharding plans than the baselines, while it takes slightly longer than \baselinedomain{} and DreamShard due to taking optimization steps rather than building a solution from a heuristic feature or reinforcement learned policy. \ours{}-\trainvershort{} obtains lower latency solutions in about the same time as DreamShard with a slight runtime increase from SCIP. Lastly, \ours{}-\hybridvershort{} obtains the best solutions and has runtime comparable to \ours{}-\optvershort{}.
In smaller instances ($T\leq40$), \ours{}-\trainvershort{} finds better solutions than its impromptu counterpart, \ours{}-\optvershort{}, likely by escaping local optima by training on a variety of examples.
For larger instances with more tables available for placement, \ours{}-\optvershort{} performs better by optimizing for the test instances as opposed to \ours{}-\trainvershort{} which only uses training data. Using \ours{}-\hybridvershort{}, we obtain the best solutions but incur the upfront pretraining cost and the deployment-time optimization cost.
%




\subsection{Inverse Photonic Design}

\begin{figure*}[ht!]
\centering
\includegraphics[width=\textwidth]{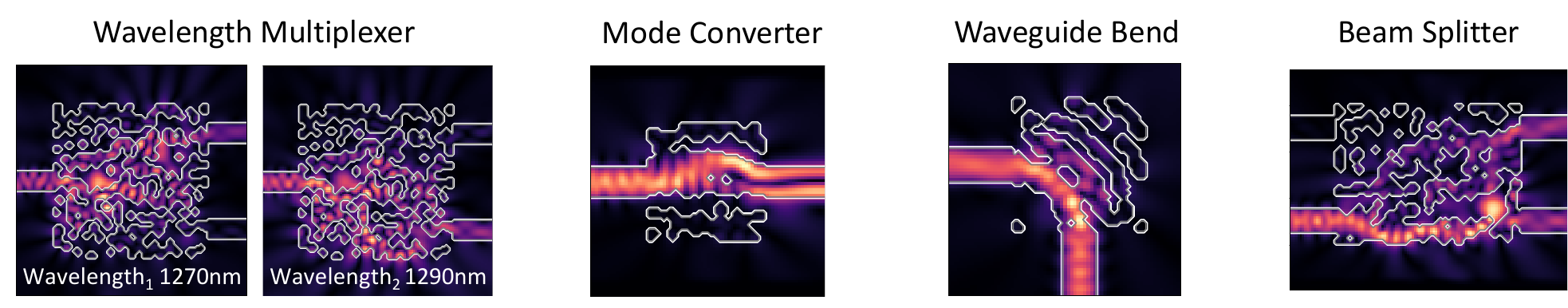}
\caption{\small Inverse photonic design settings from the ceviche challenges \cite{schubert_inverse_2022} along with \ours{}-\optvershort{} solution designs and wavelength intensities. Light is fed in on the left and is routed at desired intensities to the output by designing the intermediate region. In the Wavelength Multiplexer setting, two wavelengths of interest are visualized as they are routed to different locations.}
\label{fig:inv_photo_settings}
\end{figure*}





Photonic devices play an essential role in high-speed communication \citep{marpaung2019microwave}, quantum computing \citep{arrazola2021quantum}, and machine learning hardware acceleration \citep{wetzstein2020photomlaccel}. The photonic components can be encoded as a binary 2D grid, with each cell being filled or void. There are constraints on which binary patterns are physically manufacturable: only those that can be drawn by a physical brush instrument with a specific cross shape can be manufactured.
It remains challenging to find manufacturable designs that satisfy design specifications like splitting beams of light. An example solution developed by \ours{} is shown in Figure \ref{fig:inv_photo_solution}: coming from the top, beams are routed to the left or right, depending on wavelength. The solution is also manufacturable: a 3-by-3 brush cross can fit in all filled and void space. 
Given the design, existing work~\citep{hughes2019ceviche} enables differentiation of the design misspecification cost, evaluated as how far off the transmission intensity of the wavelengths are from the desired output locations, with zero design loss meaning that the specification is satisfied. Researchers also develop the Ceviche Challenges~\citep{schubert_inverse_2022} a standard benchmark of inverse photonic design problems. Formally, a feasible design is a rectangle of pixels which are either filled or void where both the filled and void pixels can be expressed as a unions of the brush shape. Please see ~\citep{schubert_inverse_2022} for an in depth description of the nonlinear objective and feasible region.

\begin{figure*}[ht!]
\sbox\twosubbox{%
    \includegraphics[height=4.1cm]{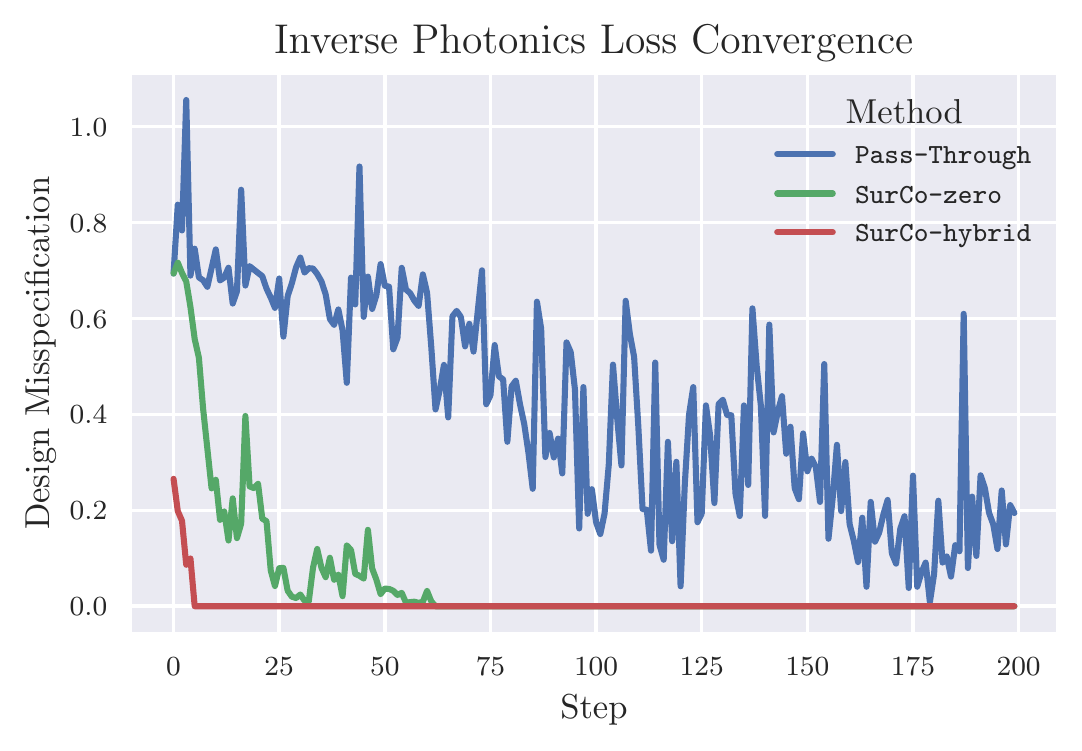}%
    \includegraphics[height=4.1cm]{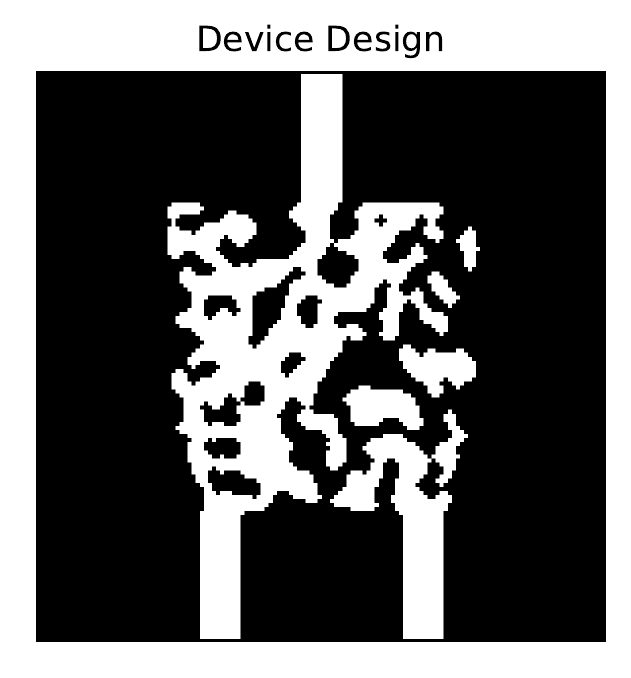}%
    \includegraphics[height=4.1cm]{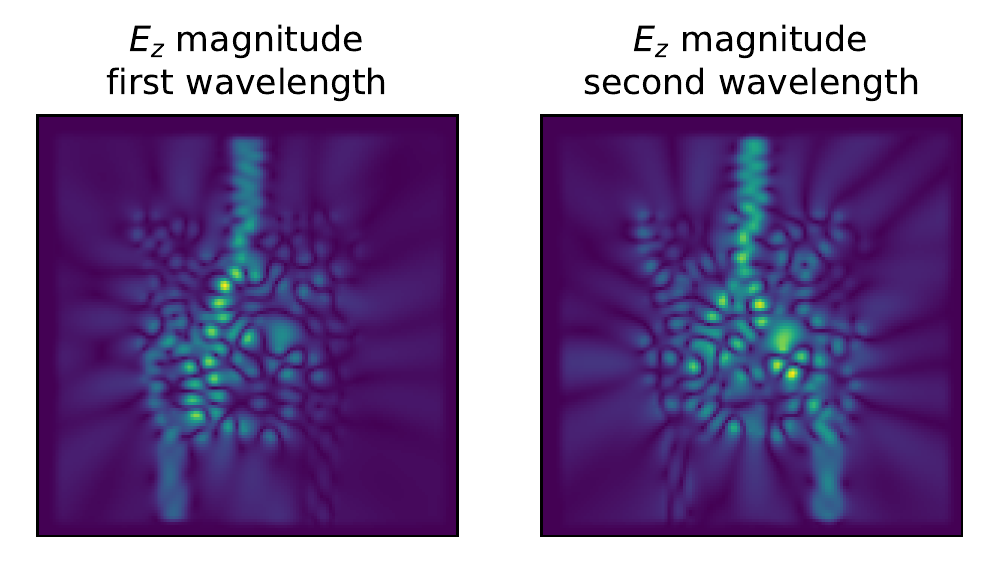}%
}
\setlength{\twosubht}{\ht\twosubbox}

\centering

\subcaptionbox{\label{fig:inv_photo_convergence}Loss Convergence}{%
\includegraphics[height=4.cm]{plots/inverse_photonics_convergence.pdf}%
}
\quad
\subcaptionbox{\label{fig:inv_photo_solution}Device}{%
\includegraphics[height=3.8cm]{plots/splitter_design.pdf}%
}
\quad
\subcaptionbox{\label{fig:inv_photo_wavelength}Wave Mangitude}{%
\includegraphics[height=3.8cm]{plots/splitter_wavelength.pdf}%
}

\caption{\small Inverse photonic design convergence example \citep{schubert_inverse_2022}. In (a), \ours{}-\optvershort{} smoothly lowers the loss while the pass-through baseline converges noisily. Also, \ours{}-\hybridvershort{} quickly fine-tunes an already high-quality solution. (b) visualizes the \ours{}-\optvershort{} solution and (c) visualizes the two wavelengths of interest which are successfully routed from the top to the bottom.}
\label{fig:inv_photo_example}
\end{figure*}

\textbf{Settings.} We compare our approaches against the \texttt{Pass-Through} method \citep{schubert_inverse_2022} on randomly generated instances of the four types of problems in \cite{schubert_inverse_2022}:
Waveguide Bend, Mode Converter, Wavelengths Division Multiplexer, and Beam Splitter. We generate 50 instances in each setting (25 training/25 test), randomly sampling the location of input and output waveguides, or ``pipes'' where we are taking in light and desire light to output. We fix the wavelengths themselves and so the problem description $\vy$ contains an image description of the problem instance, where each pixel is either ``fixed'' or ``designable''. Further generation details are in the appendix. 
We evaluated several algorithms described in the appendix, such as genetic algorithms and derivative-free optimization, which failed to find physically feasible solutions. We consider two wavelengths (1270nm/1290nm), and optimize at a resolution of 40nm, visualizing the test results in Fig. \ref{fig:inv_photo_results}.

\textbf{Results.} Fig. \ref{fig:inv_photo_results}, \ours{}-\optvershort{} consistently finds as many or more valid devices compared to the \baselinemartin{} baseline \citep{schubert_inverse_2022}. Additionally, since the on-the-fly solvers stop when they either find a valid solution, or reach a maximum of 200 steps, the runtime of \ours{}-\optvershort{} is slightly lower than
the \baselinemartin{} baseline. \ours{}-\trainvershort{} obtains similar success rates as \baselinemartin{} while taking two orders of magnitude less time as it does not require expensive impromptu optimization, making \ours{}-\trainvershort{} a promising approach for large-scale settings or when solving many slightly-varied instances. Lastly, \ours{}-\hybridvershort{} performs best in terms of solution loss, finding valid solutions more often than the other approaches. It also takes less runtime than the other on-the-fly approaches since it is able to reach valid solutions faster, although it still requires optimization on-the-fly so it takes longer than \ours{}-\trainvershort{}.
We visualize impromptu solver convergence in Fig. \ref{fig:inv_photo_convergence} where \ours{}-\optvershort{} has smoother and faster convergence than \texttt{Pass-Through}.

\subsection{Nonlinear Route Planning}

\begin{figure}[]
\centering
\scriptsize
\includegraphics[width=0.6\columnwidth]{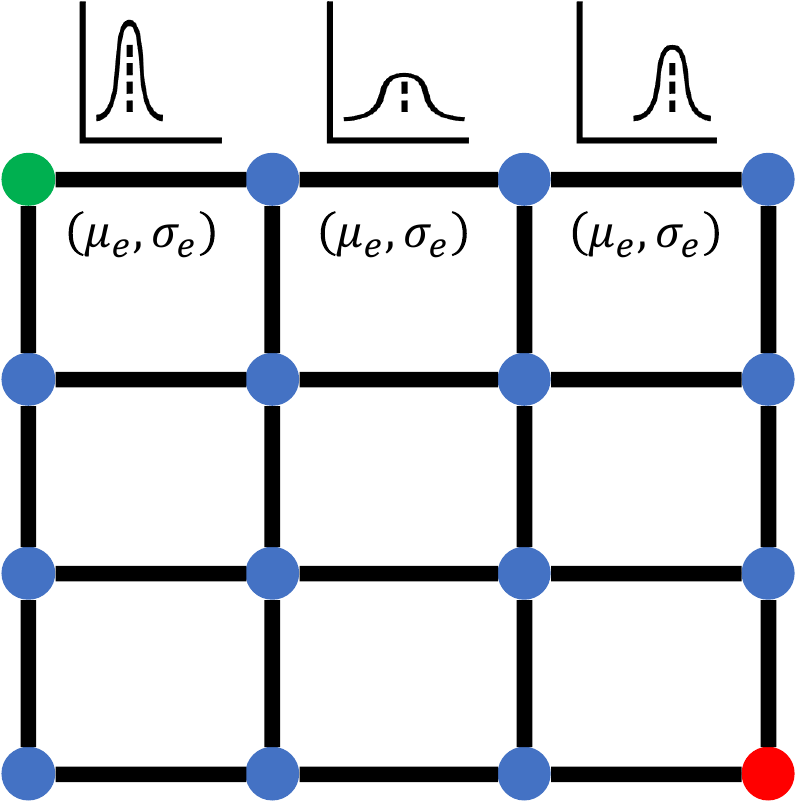}
\caption{\small Nonlinear route planning visualization. The goal is to route from the top left to bottom right corner, with the edge weights being normally distributed, maximizing the probability of arriving before a set deadline.}
\label{fig:nonlinear_shortest_path}
\end{figure}

Nonlinear route planning can arise where one wants to maximize the probability of arrival before a set time in graphs with random edges \citep{fan2005arrivingontime,nikolova2006stochastic,lim2013practical}. 
These problems occur in risk-aware settings such as emergency services operators who need to maximize the probability of arriving before a critical time, or where driver reward is determined by deadlines. 

Given a graph $G$ with edge lengths coming from a random distribution, a pair of source and destination nodes $s,t$, and a time limit $T$ that we would like to arrive before, we select a feasible $s-t$ path $P_{s,t}$ that maximizes the probability of arriving before the deadline $P[\text{length}(P_{s,t}) \leq T]$. If we assume that edge times are distributed according to a random normal distribution $t_e\sim \gN(\mu_e, \sigma_e^2)$, then we could write the objective as maximizing $f(x;y)=\Phi\left((T-\sum_{e\in P_{s,t}}\mu_e)/\sqrt{\sum_{e\in P_{s,t}}\sigma^2_e}\right)$, with $\Phi$ being the cumulative distribution function of a standard Gaussian distribution, with the feasible region $\Omega(y)$ being the set of $s-t$ paths in the graph from origin to destination. Explicitly, the problem parameters $\vy$ are the graph $G$, source and destination nodes $s,t$, time limit $T$, and the edge weight distributions specified by the edge means and variances $\mu_e, \sigma_e^2$.
We only consider the zero-shot setting without training examples since we need to solve the problem on-the-fly.
\ours{} trains surrogate edge costs $\hat{c}_e$ and finds the shortest path using Bellman-Ford \citep{bellman1958routing}, and differentiate using blackbox differentiation \citep{poganvcic2019diffbb}.

\textbf{Settings.}
We run on a 5x5 grid graph with 25 draws of edge parameters $\mu_e \sim U(0.1,1)$ and $\sigma^2_e\sim U(0.1,0.3)*(1-\mu_e)$, with $U(a,b)$ being the uniform random distribution between $a$ and $b$. We have deadline settings based on the length of the least expected time path (LET) which is simply the shortest path using $\mu_e$ as weights. We use loose, normal, and tight deadlines of 1.1 LET, 1 LET, and 0.9 LET respectively. The source and destination are oppose corners of the grid graph.

\begin{figure}[ht!]
\centering
    \includegraphics[width=\linewidth]{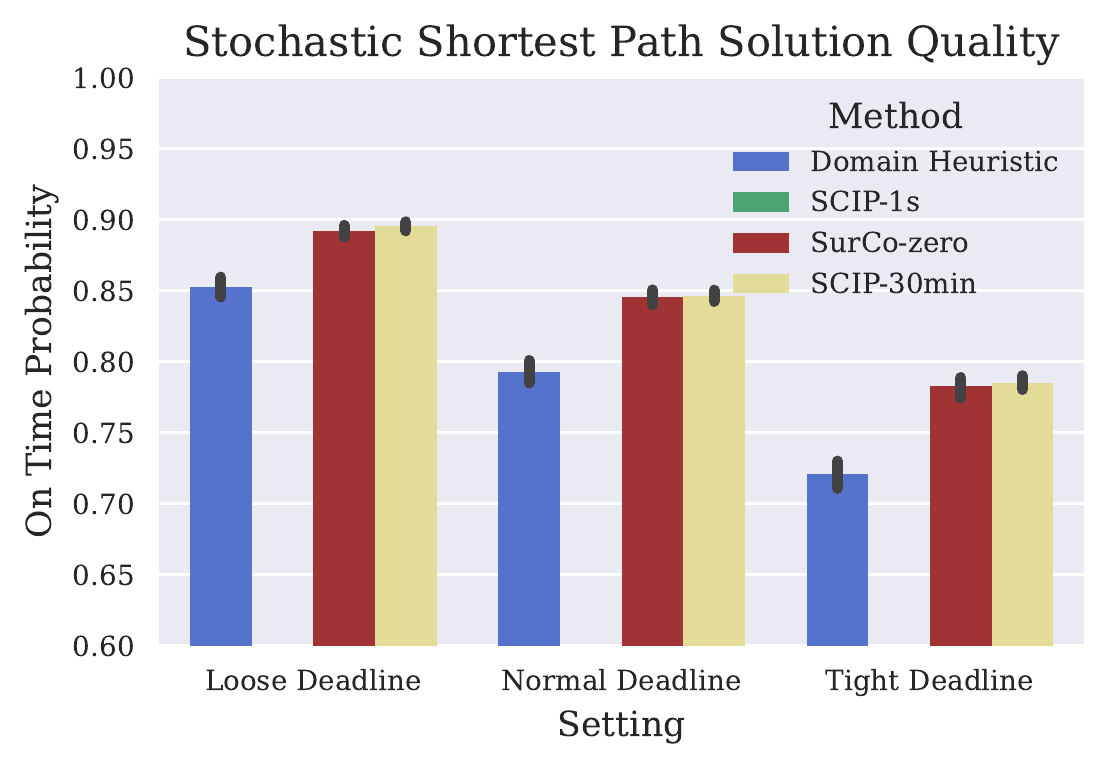}%
    \caption{\small Comparison of nonlinear route planning probability of arriving on time. We compare against a domain heuristic \citep{nikolova2006stochastic} and SCIP \citep{achterberg2009scip}. \ours{}-\optvershort{} outperforms the domain heuristic, and is similar to SCIP using less time. SCIP-1s fails to find feasible solutions.}
\label{fig:stochastic_shortest_path}
\end{figure}

\textbf{Results.} Fig. \ref{fig:stochastic_shortest_path}, we compare \ours{}-\optvershort{} against a domain-specific approach that minimizes a linear combination of mean and variance \citep{nikolova2006stochastic}, and SCIP \citep{achterberg2009scip}. In this setting, we focus on the zero-shot performance of \ours{}, comparing it against two other zero-shot approaches. Furthermore, here we are able to encode the objective analytically into SCIP whereas the objectives of the other settings do not have readily-encodeable formulations, relying on neural networks or physical simulation. Since \ours{}-\optvershort{} and the domain approach take much less than 1 second, we use SCIP-1s and find that SCIP cannot find feasible solutions at that time scale. SCIP-30min demonstrates how well a general-purpose method can do given enough time, with SCIP timing out on all instances. We also find that \ours{}-\optvershort{} is able to obtain comparable solutions to SCIP-30min. Furthermore, \ours{}-\optvershort{} consistently outperforms the domain heuristic, finding paths that reach the deadline with $4.5\%, 6.5\%, 8.5\%$ times higher success rates in loose, normal, and tight deadlines. Finally, we found only 2 instances where the domain heuristic beat \ours{}-\optvershort{}.

\section{Related Work}

\paragraph{Differentiable Optimization}
OptNet \citep{amos2017optnet} proposed implicitly differentiating through KKT conditions, a set of linear equations that determine the optimal solution. Followup work differentiated through linear programs \citep{wilder2019melding}, submodular optimization problems \citep{djolonga2017submodular, wilder2019melding}, cone programs \citep{agrawal2019cvxpylayers,agrawal2019cone}, MaxSAT \citep{wang2019satnet}, Mixed Integer Linear Programming \citep{ferber2020mipaal,mandi2020spomilp}, Integer Linear Programming \citep{mandi2020spomilp}, dynamic programming \cite{demirovic2020dynamic}, blackbox discrete linear optimizers \citep{poganvcic2019diffbb,rolinek2020rankmetric,rolinek2020graphmatching}, maximum likelihood estimation \citep{niepert2021imle}, kmeans clustering \citep{wilder2019graphcombopt}, knapsack \citep{guler2022divide,demirovic2019exhaustive}, the cross-entropy method \citep{amos2020cem}, Nonlinear Least Squares \citep{pineda2022theseus}, SVM training \citep{lee2019metasvm}, and combining LP variables \cite{wang2020surrogates}. \ours{} can leverage these differentiable surrogates for different problem domains. 

\paragraph{Task Based Learning}
Task-based learning solves distributions of linear or quadratic optimization problems with the true objective hidden at test time but available for training \citep{elmachtoub2022spo, donti2017task,el2019generalizationspo,liu2021riskspo,hu2022fastconvergencespo}.
\citep{donti2021dc3} predicts and corrects solutions for continuous nonlinear optimization.
{Bayesian optimization (BO) \citep{shahriari2016bayesopt}, optimizes blackbox functions by approximating the objective with a learned model that can be optimized over. Recent work optimizes individual instances over discrete spaces like hypercubes \citep{baptista2018bayesianco}, graphs \citep{deshwal2021mercer}, and MILP \citep{papalexopoulos2022bomilpnn}. Data reuse from previous runs is proposed to optimize multiple correlated instances \citep{swersky2013multitask, feurer2018scalablemetalearning}. However, the surrogate Gaussian Process (GP) models are memory and time intensive in high-dimensional settings. Recent work has addressed GP scalability via gradient updates \citep{ament2022scalablebo}; however, it is unclear whether GP can scale in conjunction with combinatorial solvers.}
%
Machine learning is also used to guide combinatorial algorithms. Several approaches produce combinatorial solutions \citep{zhang1995rljobshop,khalil2017graphco,kool2018attention,nazari2018rlvrp,zha2022autoshard,zha2022dreamshard}. Here, approaches are limited to simple feasible regions by iteratively building solutions for problems like routing, assignment, or covering. However, these approaches fail to handle more complex constraints. Other approaches set parameters that improve solver runtime \citep{khalil2016learningtobranch,bengio2021ml4co}. Similarly, a neural diving approach has been proposed for finding fast MILP solutions \cite{nair2020neuraldiving}. This approach requires iteratively solving a subproblem which in the nonlinear setting may still be hard to solve or even encode.

\paragraph{Learning Latent Space for Optimization}
We learn latent linear objectives to optimize nonlinear functions while other approaches learn latent embeddings for faster solving. FastMap \citep{faloutsos1995fastmap} learns latent object embeddings for efficient search, and variants of FastMap are used in graph optimization and shortest path \citep{cohen2018fastmapsp,hu2022fastconvergencespo,li2019fastmapgraphopt}. \cite{wang2020lamcts,wang2021sample,yang2021learning,zhao2022multiobjective} use Monte Carlo Tree Search to perform single and multi-objective optimization by learning to split the search space.

\paragraph{Mixed Integer Nonlinear Programming (MINLP)}
\ours{}-\optvershort{} operates as a MINLP solver, optimizing nonlinear and nonconvex objectives over discrete linear feasible regions. Specialized solvers handle some MINLP variants \citep{burer2012minlpnonconv, belotti2013minlp}; however, scalability in nonconvex settings usually requires problem-specific techniques like piecewise linear approximation, objective convexification, or exploiting special structure. 

\section{Conclusion}

We introduced \ours{}, a method for learning linear surrogates for combinatorial nonlinear optimization problems. \ours{} learns linear objective coefficients for a surrogate solver which results in solutions that minimize the nonlinear loss via gradient descent. At its core, \ours{} differentiates through the surrogate solver which maps the predicted coefficients to a combinatorially feasible solution, combining the flexibility of gradient-based optimization with the structure of combinatorial solvers. Our theoretical intuition for \ours{} poses promising directions for future work in proving convergence guarantees or generalization bounds. Additionally,  improvements of \ours{} may enable scalable solving for settings in stochastic optimization, game theory, combinatorial reinforcement learning, and more. We presented three variants of \ours{}, \ours{}-\optver{} which optimizes individual instances, \ours{}-\trainver{} which trains a coefficient prediction model offline, and \ours{}-\hybridver{} which fine-tunes the coefficients predicted by \ours{}-\trainver{} on individual test instances. {While \ours{}'s performance is somewhat limited to binary problems due to the lack of interior integer points, we find that many real-world domains operate on binary decision variables. We evaluated variants of \ours{}} against the state-of-the-art approaches on three domains, with two used in industry, obtaining better solution quality for similar or better runtime in the embedding table sharding domain, quickly identifying viable photonic devices, and finding successful routes in stochastic path planning. Overall, \ours{} trains linear surrogate coefficients to point the solver towards high-quality solutions, becoming a general-purpose method that aims to tackle a broad class of combinatorial problems with nonlinear objectives when off-the-shelf solvers fail.
\section*{Acknowledgements.} 
This paper reports on research done while Aaron Ferber, Taoan Huang, and Daochen Zha were interns at Meta AI (FAIR). The research at the University of Southern California was supported by the National Science Foundation (NSF) under grant number 2112533. We also thank the anonymous reviewers for helpful feedback.

\newpage
\bibliography{09_bibliography}

\begin{thebibliography}{98}
\providecommand{\natexlab}[1]{#1}
\providecommand{\url}[1]{\texttt{#1}}
\expandafter\ifx\csname urlstyle\endcsname\relax
  \providecommand{\doi}[1]{doi: #1}\else
  \providecommand{\doi}{doi: \begingroup \urlstyle{rm}\Url}\fi

\bibitem[Achterberg(2009)]{achterberg2009scip}
Achterberg, T.
\newblock Scip: solving constraint integer programs.
\newblock \emph{Mathematical Programming Computation}, 1\penalty0 (1):\penalty0
  1--41, 2009.

\bibitem[Agrawal et~al.(2019{\natexlab{a}})Agrawal, Amos, Barratt, Boyd,
  Diamond, and Kolter]{agrawal2019cvxpylayers}
Agrawal, A., Amos, B., Barratt, S., Boyd, S., Diamond, S., and Kolter, J.~Z.
\newblock Differentiable convex optimization layers.
\newblock \emph{Advances in neural information processing systems}, 32,
  2019{\natexlab{a}}.

\bibitem[Agrawal et~al.(2019{\natexlab{b}})Agrawal, Barratt, Boyd, Busseti, and
  Moursi]{agrawal2019cone}
Agrawal, A., Barratt, S., Boyd, S., Busseti, E., and Moursi, W.~M.
\newblock Differentiating through a cone program.
\newblock \emph{J. Appl. Numer. Optim}, 1\penalty0 (2):\penalty0 107--115,
  2019{\natexlab{b}}.

\bibitem[Ament \& Gomes(2022)Ament and Gomes]{ament2022scalablebo}
Ament, S.~E. and Gomes, C.~P.
\newblock Scalable first-order bayesian optimization via structured automatic
  differentiation.
\newblock In \emph{International Conference on Machine Learning}, pp.\
  500--516. PMLR, 2022.

\bibitem[Amos \& Kolter(2017)Amos and Kolter]{amos2017optnet}
Amos, B. and Kolter, J.~Z.
\newblock Optnet: Differentiable optimization as a layer in neural networks.
\newblock In \emph{International Conference on Machine Learning}, pp.\
  136--145. PMLR, 2017.

\bibitem[Amos \& Yarats(2020)Amos and Yarats]{amos2020cem}
Amos, B. and Yarats, D.
\newblock The differentiable cross-entropy method.
\newblock In \emph{International Conference on Machine Learning}, pp.\
  291--302. PMLR, 2020.

\bibitem[Arrazola et~al.(2021)Arrazola, Bergholm, Br{\'a}dler, Bromley,
  Collins, Dhand, Fumagalli, Gerrits, Goussev, Helt,
  et~al.]{arrazola2021quantum}
Arrazola, J.~M., Bergholm, V., Br{\'a}dler, K., Bromley, T.~R., Collins, M.~J.,
  Dhand, I., Fumagalli, A., Gerrits, T., Goussev, A., Helt, L.~G., et~al.
\newblock Quantum circuits with many photons on a programmable nanophotonic
  chip.
\newblock \emph{Nature}, 591\penalty0 (7848):\penalty0 54--60, 2021.

\bibitem[Ban \& Rudin(2019)Ban and Rudin]{ban2019big}
Ban, G.-Y. and Rudin, C.
\newblock The big data newsvendor: Practical insights from machine learning.
\newblock \emph{Operations Research}, 67\penalty0 (1):\penalty0 90--108, 2019.

\bibitem[Baptista \& Poloczek(2018)Baptista and
  Poloczek]{baptista2018bayesianco}
Baptista, R. and Poloczek, M.
\newblock Bayesian optimization of combinatorial structures.
\newblock In \emph{International Conference on Machine Learning}, pp.\
  462--471. PMLR, 2018.

\bibitem[Bellman(1958)]{bellman1958routing}
Bellman, R.
\newblock On a routing problem.
\newblock \emph{Quarterly of applied mathematics}, 16\penalty0 (1):\penalty0
  87--90, 1958.

\bibitem[Belotti et~al.(2013)Belotti, Kirches, Leyffer, Linderoth, Luedtke, and
  Mahajan]{belotti2013minlp}
Belotti, P., Kirches, C., Leyffer, S., Linderoth, J., Luedtke, J., and Mahajan,
  A.
\newblock Mixed-integer nonlinear optimization.
\newblock \emph{Acta Numerica}, 22:\penalty0 1--131, 2013.

\bibitem[Bengio et~al.(2021)Bengio, Lodi, and Prouvost]{bengio2021ml4co}
Bengio, Y., Lodi, A., and Prouvost, A.
\newblock Machine learning for combinatorial optimization: a methodological
  tour d’horizon.
\newblock \emph{European Journal of Operational Research}, 290\penalty0
  (2):\penalty0 405--421, 2021.

\bibitem[Berthet et~al.(2020)Berthet, Blondel, Teboul, Cuturi, Vert, and
  Bach]{berthet2020diffperturb}
Berthet, Q., Blondel, M., Teboul, O., Cuturi, M., Vert, J.-P., and Bach, F.
\newblock Learning with differentiable pertubed optimizers.
\newblock \emph{Advances in neural information processing systems},
  33:\penalty0 9508--9519, 2020.

\bibitem[Bradbury et~al.(2018)Bradbury, Frostig, Hawkins, Johnson, Leary,
  Maclaurin, Necula, Paszke, Vander{P}las, Wanderman-{M}ilne, and
  Zhang]{jax2018github}
Bradbury, J., Frostig, R., Hawkins, P., Johnson, M.~J., Leary, C., Maclaurin,
  D., Necula, G., Paszke, A., Vander{P}las, J., Wanderman-{M}ilne, S., and
  Zhang, Q.
\newblock {JAX}: composable transformations of {P}ython+{N}um{P}y programs,
  2018.
\newblock URL \url{http://github.com/google/jax}.

\bibitem[Burer \& Letchford(2012)Burer and Letchford]{burer2012minlpnonconv}
Burer, S. and Letchford, A.~N.
\newblock Non-convex mixed-integer nonlinear programming: A survey.
\newblock \emph{Surveys in Operations Research and Management Science},
  17\penalty0 (2):\penalty0 97--106, 2012.

\bibitem[Chvatal et~al.(1983)Chvatal, Chvatal, et~al.]{chvatal1983linear}
Chvatal, V., Chvatal, V., et~al.
\newblock \emph{Linear programming}.
\newblock Macmillan, 1983.

\bibitem[Cohen et~al.(2018)Cohen, Uras, Jahangiri, Arunasalam, Koenig, and
  Kumar]{cohen2018fastmapsp}
Cohen, L., Uras, T., Jahangiri, S., Arunasalam, A., Koenig, S., and Kumar,
  T.~S.
\newblock The fastmap algorithm for shortest path computations.
\newblock In \emph{IJCAI}, 2018.

\bibitem[Cozad et~al.(2014)Cozad, Sahinidis, and Miller]{cozad2014learning}
Cozad, A., Sahinidis, N.~V., and Miller, D.~C.
\newblock Learning surrogate models for simulation-based optimization.
\newblock \emph{AIChE Journal}, 60\penalty0 (6):\penalty0 2211--2227, 2014.

\bibitem[Demirovi{\'c} et~al.(2019)Demirovi{\'c}, J~Stuckey, Bailey, Chan,
  Leckie, Ramamohanarao, and Guns]{demirovic2019exhaustive}
Demirovi{\'c}, E., J~Stuckey, P., Bailey, J., Chan, J., Leckie, C.,
  Ramamohanarao, K., and Guns, T.
\newblock Predict+ optimise with ranking objectives: Exhaustively learning
  linear functions.
\newblock In \emph{Proceedings of the Twenty-Eighth International Joint
  Conference on Artificial Intelligence, IJCAI 2019, Macao, China, August
  10-16, 2019}, pp.\  1078--1085. International Joint Conferences on Artificial
  Intelligence, 2019.

\bibitem[Demirovic et~al.(2020)Demirovic, J~Stuckey, Guns, Bailey, Leckie,
  Ramamohanarao, and Chan]{demirovic2020dynamic}
Demirovic, E., J~Stuckey, P., Guns, T., Bailey, J., Leckie, C., Ramamohanarao,
  K., and Chan, J.
\newblock Dynamic programming for predict+ optimise.
\newblock In \emph{The Thirty-Fourth AAAI Conference on Artificial
  Intelligence, AAAI 2020, The Thirty-Second Innovative Applications of
  Artificial Intelligence Conference, IAAI 2020, The Tenth AAAI Symposium on
  Educational Advances in Artificial Intelligence, EAAI 2020, New York, NY,
  USA, February 7-12, 2020}, pp.\  1444--1451. AAAI Press, 2020.

\bibitem[Deshwal et~al.(2021)Deshwal, Belakaria, and Doppa]{deshwal2021mercer}
Deshwal, A., Belakaria, S., and Doppa, J.~R.
\newblock Mercer features for efficient combinatorial bayesian optimization.
\newblock \emph{Proceedings of the AAAI Conference on Artificial Intelligence},
  35\penalty0 (8):\penalty0 7210--7218, May 2021.
\newblock \doi{10.1609/aaai.v35i8.16886}.
\newblock URL \url{https://ojs.aaai.org/index.php/AAAI/article/view/16886}.

\bibitem[Djolonga \& Krause(2017)Djolonga and Krause]{djolonga2017submodular}
Djolonga, J. and Krause, A.
\newblock Differentiable learning of submodular models.
\newblock \emph{Advances in Neural Information Processing Systems}, 30, 2017.

\bibitem[Donti et~al.(2017)Donti, Amos, and Kolter]{donti2017task}
Donti, P., Amos, B., and Kolter, J.~Z.
\newblock Task-based end-to-end model learning in stochastic optimization.
\newblock \emph{Advances in neural information processing systems}, 30, 2017.

\bibitem[Donti et~al.(2021)Donti, Rolnick, and Kolter]{donti2021dc3}
Donti, P.~L., Rolnick, D., and Kolter, J.~Z.
\newblock {DC}3: A learning method for optimization with hard constraints.
\newblock In \emph{International Conference on Learning Representations}, 2021.
\newblock URL \url{https://openreview.net/forum?id=V1ZHVxJ6dSS}.

\bibitem[El~Balghiti et~al.(2019)El~Balghiti, Elmachtoub, Grigas, and
  Tewari]{el2019generalizationspo}
El~Balghiti, O., Elmachtoub, A.~N., Grigas, P., and Tewari, A.
\newblock Generalization bounds in the predict-then-optimize framework.
\newblock \emph{Advances in neural information processing systems}, 32, 2019.

\bibitem[Elmachtoub \& Grigas(2022{\natexlab{a}})Elmachtoub and
  Grigas]{elmachtoub2022smart}
Elmachtoub, A.~N. and Grigas, P.
\newblock Smart “predict, then optimize”.
\newblock \emph{Management Science}, 68\penalty0 (1):\penalty0 9--26,
  2022{\natexlab{a}}.

\bibitem[Elmachtoub \& Grigas(2022{\natexlab{b}})Elmachtoub and
  Grigas]{elmachtoub2022spo}
Elmachtoub, A.~N. and Grigas, P.
\newblock Smart “predict, then optimize”.
\newblock \emph{Management Science}, 68\penalty0 (1):\penalty0 9--26,
  2022{\natexlab{b}}.

\bibitem[Faloutsos \& Lin(1995)Faloutsos and Lin]{faloutsos1995fastmap}
Faloutsos, C. and Lin, K.-I.
\newblock Fastmap: A fast algorithm for indexing, data-mining and visualization
  of traditional and multimedia datasets.
\newblock In \emph{Proceedings of the 1995 ACM SIGMOD International Conference
  on Management of Data}, SIGMOD '95, pp.\  163–174, New York, NY, USA, 1995.
  Association for Computing Machinery.
\newblock ISBN 0897917316.
\newblock \doi{10.1145/223784.223812}.
\newblock URL \url{https://doi.org/10.1145/223784.223812}.

\bibitem[Fan et~al.(2005)Fan, Kalaba, and Moore]{fan2005arrivingontime}
Fan, Y., Kalaba, R.~E., and Moore, J.~E.
\newblock Arriving on time.
\newblock \emph{Journal of Optimization Theory and Applications}, 127:\penalty0
  497--513, 2005.

\bibitem[Ferber et~al.(2020)Ferber, Wilder, Dilkina, and
  Tambe]{ferber2020mipaal}
Ferber, A., Wilder, B., Dilkina, B., and Tambe, M.
\newblock Mipaal: Mixed integer program as a layer.
\newblock In \emph{Proceedings of the AAAI Conference on Artificial
  Intelligence}, volume~34, pp.\  1504--1511, 2020.

\bibitem[Feurer et~al.(2018)Feurer, Letham, and
  Bakshy]{feurer2018scalablemetalearning}
Feurer, M., Letham, B., and Bakshy, E.
\newblock Scalable meta-learning for bayesian optimization.
\newblock \emph{stat}, 1050\penalty0 (6), 2018.

\bibitem[Fix(1985)]{fix19851nn}
Fix, E.
\newblock \emph{Discriminatory analysis: nonparametric discrimination,
  consistency properties}, volume~1.
\newblock USAF school of Aviation Medicine, 1985.

\bibitem[Gad(2021)]{gad2021pygad}
Gad, A.~F.
\newblock Pygad: An intuitive genetic algorithm python library, 2021.

\bibitem[Gosavi et~al.(2015)]{gosavi2015simulation}
Gosavi, A. et~al.
\newblock \emph{Simulation-based optimization}.
\newblock Springer, 2015.

\bibitem[Guler et~al.(2022)Guler, Demirovi{\'c}, Chan, Bailey, Leckie, and
  Stuckey]{guler2022divide}
Guler, A.~U., Demirovi{\'c}, E., Chan, J., Bailey, J., Leckie, C., and Stuckey,
  P.~J.
\newblock A divide and conquer algorithm for predict+ optimize with non-convex
  problems.
\newblock In \emph{Proceedings of the AAAI Conference on Artificial
  Intelligence}, volume~36, pp.\  3749--3757, 2022.

\bibitem[{Gurobi Optimization, LLC}(2022)]{gurobi}
{Gurobi Optimization, LLC}.
\newblock {Gurobi Optimizer Reference Manual}, 2022.
\newblock URL \url{https://www.gurobi.com}.

\bibitem[Hu et~al.(2022)Hu, Kallus, and Mao]{hu2022fastconvergencespo}
Hu, Y., Kallus, N., and Mao, X.
\newblock Fast rates for contextual linear optimization.
\newblock \emph{Management Science}, 2022.

\bibitem[Hughes et~al.(2019)Hughes, Williamson, Minkov, and
  Fan]{hughes2019ceviche}
Hughes, T.~W., Williamson, I.~A., Minkov, M., and Fan, S.
\newblock Forward-mode differentiation of maxwell’s equations.
\newblock \emph{ACS Photonics}, 6\penalty0 (11):\penalty0 3010--3016, 2019.

\bibitem[Jumper et~al.(2021)Jumper, Evans, Pritzel, Green, Figurnov,
  Ronneberger, Tunyasuvunakool, Bates, {\v{Z}}{\'\i}dek, Potapenko,
  et~al.]{jumper2021highly}
Jumper, J., Evans, R., Pritzel, A., Green, T., Figurnov, M., Ronneberger, O.,
  Tunyasuvunakool, K., Bates, R., {\v{Z}}{\'\i}dek, A., Potapenko, A., et~al.
\newblock Highly accurate protein structure prediction with alphafold.
\newblock \emph{Nature}, 596\penalty0 (7873):\penalty0 583--589, 2021.

\bibitem[Khalil et~al.(2016)Khalil, Le~Bodic, Song, Nemhauser, and
  Dilkina]{khalil2016learningtobranch}
Khalil, E., Le~Bodic, P., Song, L., Nemhauser, G., and Dilkina, B.
\newblock Learning to branch in mixed integer programming.
\newblock In \emph{Proceedings of the AAAI Conference on Artificial
  Intelligence}, volume~30, 2016.

\bibitem[Khalil et~al.(2017)Khalil, Dai, Zhang, Dilkina, and
  Song]{khalil2017graphco}
Khalil, E., Dai, H., Zhang, Y., Dilkina, B., and Song, L.
\newblock Learning combinatorial optimization algorithms over graphs.
\newblock \emph{Advances in neural information processing systems}, 30, 2017.

\bibitem[Kool et~al.(2018)Kool, van Hoof, and Welling]{kool2018attention}
Kool, W., van Hoof, H., and Welling, M.
\newblock Attention, learn to solve routing problems!
\newblock In \emph{International Conference on Learning Representations}, 2018.

\bibitem[Korte \& Hausmann(1978)Korte and Hausmann]{korte1978analysis}
Korte, B. and Hausmann, D.
\newblock An analysis of the greedy heuristic for independence systems.
\newblock In \emph{Annals of Discrete Mathematics}, volume~2, pp.\  65--74.
  Elsevier, 1978.

\bibitem[Koziel et~al.(2021)Koziel, {\c{C}}al{\i}k, Mahouti, and
  Belen]{koziel2021accurate}
Koziel, S., {\c{C}}al{\i}k, N., Mahouti, P., and Belen, M.~A.
\newblock Accurate modeling of antenna structures by means of domain
  confinement and pyramidal deep neural networks.
\newblock \emph{IEEE Transactions on Antennas and Propagation}, 70\penalty0
  (3):\penalty0 2174--2188, 2021.

\bibitem[Land \& Doig(2010)Land and Doig]{land2010automatic}
Land, A.~H. and Doig, A.~G.
\newblock An automatic method for solving discrete programming problems.
\newblock In \emph{50 Years of Integer Programming 1958-2008}, pp.\  105--132.
  Springer, 2010.

\bibitem[Lee et~al.(2019)Lee, Maji, Ravichandran, and Soatto]{lee2019metasvm}
Lee, K., Maji, S., Ravichandran, A., and Soatto, S.
\newblock Meta-learning with differentiable convex optimization.
\newblock In \emph{Proceedings of the IEEE/CVF conference on computer vision
  and pattern recognition}, pp.\  10657--10665, 2019.

\bibitem[Li et~al.(2019)Li, Felner, Koenig, and Kumar]{li2019fastmapgraphopt}
Li, J., Felner, A., Koenig, S., and Kumar, T.~S.
\newblock Using fastmap to solve graph problems in a euclidean space.
\newblock In \emph{Proceedings of the international conference on automated
  planning and scheduling}, volume~29, pp.\  273--278, 2019.

\bibitem[Li et~al.(2021)Li, Yan, and Wu]{li2021learning}
Li, S., Yan, Z., and Wu, C.
\newblock Learning to delegate for large-scale vehicle routing.
\newblock \emph{Advances in Neural Information Processing Systems},
  34:\penalty0 26198--26211, 2021.

\bibitem[Li et~al.(2018)Li, Chen, and Koltun]{li2018combinatorial}
Li, Z., Chen, Q., and Koltun, V.
\newblock Combinatorial optimization with graph convolutional networks and
  guided tree search.
\newblock \emph{Advances in neural information processing systems}, 31, 2018.

\bibitem[Lim et~al.(2013)Lim, Sommer, Nikolova, and Rus]{lim2013practical}
Lim, S., Sommer, C., Nikolova, E., and Rus, D.
\newblock Practical route planning under delay uncertainty: Stochastic shortest
  path queries.
\newblock In \emph{Robotics: Science and Systems}, volume~8, pp.\  249--256.
  United States, 2013.

\bibitem[Liu \& Grigas(2021)Liu and Grigas]{liu2021riskspo}
Liu, H. and Grigas, P.
\newblock Risk bounds and calibration for a smart predict-then-optimize method.
\newblock \emph{Advances in Neural Information Processing Systems},
  34:\penalty0 22083--22094, 2021.

\bibitem[Liuzzi et~al.(2015)Liuzzi, Lucidi, and Rinaldi]{liuzzi2015dfl}
Liuzzi, G., Lucidi, S., and Rinaldi, F.
\newblock Derivative-free methods for mixed-integer constrained optimization
  problems.
\newblock \emph{Journal of Optimization Theory and Applications}, 164\penalty0
  (3):\penalty0 933--965, 2015.

\bibitem[Mandi et~al.(2020)Mandi, Stuckey, Guns, et~al.]{mandi2020spomilp}
Mandi, J., Stuckey, P.~J., Guns, T., et~al.
\newblock Smart predict-and-optimize for hard combinatorial optimization
  problems.
\newblock In \emph{Proceedings of the AAAI Conference on Artificial
  Intelligence}, volume~34, pp.\  1603--1610, 2020.

\bibitem[Marpaung et~al.(2019)Marpaung, Yao, and
  Capmany]{marpaung2019microwave}
Marpaung, D., Yao, J., and Capmany, J.
\newblock Integrated microwave photonics.
\newblock \emph{Nature photonics}, 13\penalty0 (2):\penalty0 80--90, 2019.

\bibitem[Mazyavkina et~al.(2021)Mazyavkina, Sviridov, Ivanov, and
  Burnaev]{mazyavkina2021reinforcement}
Mazyavkina, N., Sviridov, S., Ivanov, S., and Burnaev, E.
\newblock Reinforcement learning for combinatorial optimization: A survey.
\newblock \emph{Computers \& Operations Research}, 134:\penalty0 105400, 2021.

\bibitem[Mirhoseini et~al.(2021)Mirhoseini, Goldie, Yazgan, Jiang, Songhori,
  Wang, Lee, Johnson, Pathak, Nazi, et~al.]{mirhoseini2021chipdesign}
Mirhoseini, A., Goldie, A., Yazgan, M., Jiang, J.~W., Songhori, E., Wang, S.,
  Lee, Y.-J., Johnson, E., Pathak, O., Nazi, A., et~al.
\newblock A graph placement methodology for fast chip design.
\newblock \emph{Nature}, 594\penalty0 (7862):\penalty0 207--212, 2021.

\bibitem[Nagai et~al.(2020)Nagai, Akashi, and Sugino]{nagai2020completing}
Nagai, R., Akashi, R., and Sugino, O.
\newblock Completing density functional theory by machine learning hidden
  messages from molecules.
\newblock \emph{npj Computational Materials}, 6\penalty0 (1):\penalty0 1--8,
  2020.

\bibitem[Nair et~al.(2020)Nair, Bartunov, Gimeno, Von~Glehn, Lichocki, Lobov,
  O'Donoghue, Sonnerat, Tjandraatmadja, Wang, et~al.]{nair2020neuraldiving}
Nair, V., Bartunov, S., Gimeno, F., Von~Glehn, I., Lichocki, P., Lobov, I.,
  O'Donoghue, B., Sonnerat, N., Tjandraatmadja, C., Wang, P., et~al.
\newblock Solving mixed integer programs using neural networks.
\newblock \emph{arXiv preprint arXiv:2012.13349}, 2020.

\bibitem[Naumov et~al.(2019)Naumov, Mudigere, Shi, Huang, Sundaraman, Park,
  Wang, Gupta, Wu, Azzolini, Dzhulgakov, Mallevich, Cherniavskii, Lu,
  Krishnamoorthi, Yu, Kondratenko, Pereira, Chen, Chen, Rao, Jia, Xiong, and
  Smelyanskiy]{DLRM19}
Naumov, M., Mudigere, D., Shi, H.~M., Huang, J., Sundaraman, N., Park, J.,
  Wang, X., Gupta, U., Wu, C., Azzolini, A.~G., Dzhulgakov, D., Mallevich, A.,
  Cherniavskii, I., Lu, Y., Krishnamoorthi, R., Yu, A., Kondratenko, V.,
  Pereira, S., Chen, X., Chen, W., Rao, V., Jia, B., Xiong, L., and
  Smelyanskiy, M.
\newblock Deep learning recommendation model for personalization and
  recommendation systems.
\newblock \emph{CoRR}, abs/1906.00091, 2019.
\newblock URL \url{https://arxiv.org/abs/1906.00091}.

\bibitem[Nazari et~al.(2018)Nazari, Oroojlooy, Snyder, and
  Tak{\'a}c]{nazari2018rlvrp}
Nazari, M., Oroojlooy, A., Snyder, L., and Tak{\'a}c, M.
\newblock Reinforcement learning for solving the vehicle routing problem.
\newblock \emph{Advances in neural information processing systems}, 31, 2018.

\bibitem[Niepert et~al.(2021)Niepert, Minervini, and
  Franceschi]{niepert2021imle}
Niepert, M., Minervini, P., and Franceschi, L.
\newblock Implicit mle: backpropagating through discrete exponential family
  distributions.
\newblock \emph{Advances in Neural Information Processing Systems},
  34:\penalty0 14567--14579, 2021.

\bibitem[Nikolova et~al.(2006)Nikolova, Kelner, Brand, and
  Mitzenmacher]{nikolova2006stochastic}
Nikolova, E., Kelner, J.~A., Brand, M., and Mitzenmacher, M.
\newblock Stochastic shortest paths via quasi-convex maximization.
\newblock In \emph{European Symposium on Algorithms}, pp.\  552--563. Springer,
  2006.

\bibitem[Papalexopoulos et~al.(2022)Papalexopoulos, Tjandraatmadja, Anderson,
  Vielma, and Belanger]{papalexopoulos2022bomilpnn}
Papalexopoulos, T.~P., Tjandraatmadja, C., Anderson, R., Vielma, J.~P., and
  Belanger, D.
\newblock Constrained discrete black-box optimization using mixed-integer
  programming.
\newblock In Chaudhuri, K., Jegelka, S., Song, L., Szepesvari, C., Niu, G., and
  Sabato, S. (eds.), \emph{Proceedings of the 39th International Conference on
  Machine Learning}, volume 162 of \emph{Proceedings of Machine Learning
  Research}, pp.\  17295--17322. PMLR, 17--23 Jul 2022.
\newblock URL \url{https://proceedings.mlr.press/v162/papalexopoulos22a.html}.

\bibitem[Paszke et~al.(2019)Paszke, Gross, Massa, Lerer, Bradbury, Chanan,
  Killeen, Lin, Gimelshein, Antiga, Desmaison, Kopf, Yang, DeVito, Raison,
  Tejani, Chilamkurthy, Steiner, Fang, Bai, and Chintala]{pytorch}
Paszke, A., Gross, S., Massa, F., Lerer, A., Bradbury, J., Chanan, G., Killeen,
  T., Lin, Z., Gimelshein, N., Antiga, L., Desmaison, A., Kopf, A., Yang, E.,
  DeVito, Z., Raison, M., Tejani, A., Chilamkurthy, S., Steiner, B., Fang, L.,
  Bai, J., and Chintala, S.
\newblock Pytorch: An imperative style, high-performance deep learning library.
\newblock In Wallach, H., Larochelle, H., Beygelzimer, A., d\textquotesingle
  Alch\'{e}-Buc, F., Fox, E., and Garnett, R. (eds.), \emph{Advances in Neural
  Information Processing Systems 32}, pp.\  8024--8035. Curran Associates,
  Inc., 2019.

\bibitem[Pineda et~al.(2022)Pineda, Fan, Monge, Venkataraman, Sodhi, Chen,
  Ortiz, DeTone, Wang, Anderson, et~al.]{pineda2022theseus}
Pineda, L., Fan, T., Monge, M., Venkataraman, S., Sodhi, P., Chen, R.~T.,
  Ortiz, J., DeTone, D., Wang, A., Anderson, S., et~al.
\newblock Theseus: A library for differentiable nonlinear optimization.
\newblock \emph{Advances in Neural Information Processing Systems},
  35:\penalty0 3801--3818, 2022.

\bibitem[Pogan{\v{c}}i{\'c} et~al.(2019)Pogan{\v{c}}i{\'c}, Paulus, Musil,
  Martius, and Rolinek]{poganvcic2019diffbb}
Pogan{\v{c}}i{\'c}, M.~V., Paulus, A., Musil, V., Martius, G., and Rolinek, M.
\newblock Differentiation of blackbox combinatorial solvers.
\newblock In \emph{International Conference on Learning Representations}, 2019.

\bibitem[Rapin \& Teytaud(2018)Rapin and Teytaud]{nevergrad}
Rapin, J. and Teytaud, O.
\newblock {Nevergrad - A gradient-free optimization platform}.
\newblock \url{https://GitHub.com/FacebookResearch/Nevergrad}, 2018.

\bibitem[Reingold \& Tarjan(1981)Reingold and Tarjan]{reingold1981greedy}
Reingold, E.~M. and Tarjan, R.~E.
\newblock On a greedy heuristic for complete matching.
\newblock \emph{SIAM Journal on Computing}, 10\penalty0 (4):\penalty0 676--681,
  1981.

\bibitem[Rol{\'\i}nek et~al.(2020{\natexlab{a}})Rol{\'\i}nek, Musil, Paulus,
  Vlastelica, Michaelis, and Martius]{rolinek2020rankmetric}
Rol{\'\i}nek, M., Musil, V., Paulus, A., Vlastelica, M., Michaelis, C., and
  Martius, G.
\newblock Optimizing rank-based metrics with blackbox differentiation.
\newblock In \emph{Proceedings of the IEEE/CVF Conference on Computer Vision
  and Pattern Recognition}, pp.\  7620--7630, 2020{\natexlab{a}}.

\bibitem[Rol{\'\i}nek et~al.(2020{\natexlab{b}})Rol{\'\i}nek, Swoboda, Zietlow,
  Paulus, Musil, and Martius]{rolinek2020graphmatching}
Rol{\'\i}nek, M., Swoboda, P., Zietlow, D., Paulus, A., Musil, V., and Martius,
  G.
\newblock Deep graph matching via blackbox differentiation of combinatorial
  solvers.
\newblock In \emph{European Conference on Computer Vision}, pp.\  407--424.
  Springer, 2020{\natexlab{b}}.

\bibitem[Ruder(2016)]{ruder2016overview}
Ruder, S.
\newblock An overview of gradient descent optimization algorithms.
\newblock \emph{arXiv preprint arXiv:1609.04747}, 2016.

\bibitem[Schubert et~al.(2022)Schubert, Cheung, Williamson, Spyra, and
  Alexander]{schubert_inverse_2022}
Schubert, M.~F., Cheung, A. K.~C., Williamson, I. A.~D., Spyra, A., and
  Alexander, D.~H.
\newblock Inverse design of photonic devices with strict foundry fabrication
  constraints.
\newblock \emph{ACS Photonics}, 9\penalty0 (7):\penalty0 2327--2336, 2022.
\newblock \doi{10.1021/acsphotonics.2c00313}.

\bibitem[Sethi et~al.(2022)Sethi, Acun, Agarwal, Kozyrakis, Trippel, and
  Wu]{sethi2022recshard}
Sethi, G., Acun, B., Agarwal, N., Kozyrakis, C., Trippel, C., and Wu, C.-J.
\newblock Recshard: statistical feature-based memory optimization for
  industry-scale neural recommendation.
\newblock In \emph{Proceedings of the 27th ACM International Conference on
  Architectural Support for Programming Languages and Operating Systems}, pp.\
  344--358, 2022.

\bibitem[Shahriari et~al.(2016)Shahriari, Swersky, Wang, Adams, and
  de~Freitas]{shahriari2016bayesopt}
Shahriari, B., Swersky, K., Wang, Z., Adams, R.~P., and de~Freitas, N.
\newblock Taking the human out of the loop: A review of bayesian optimization.
\newblock \emph{Proceedings of the IEEE}, 104\penalty0 (1):\penalty0 148--175,
  2016.
\newblock \doi{10.1109/JPROC.2015.2494218}.

\bibitem[Simon(2013)]{simon2013evolutionary}
Simon, D.
\newblock \emph{Evolutionary optimization algorithms}.
\newblock John Wiley \& Sons, 2013.

\bibitem[Steiner et~al.(2021)Steiner, Cummins, He, and
  Leather]{steiner2021value}
Steiner, B., Cummins, C., He, H., and Leather, H.
\newblock Value learning for throughput optimization of deep learning
  workloads.
\newblock In Smola, A., Dimakis, A., and Stoica, I. (eds.), \emph{Proceedings
  of Machine Learning and Systems}, volume~3, pp.\  323--334, 2021.
\newblock URL
  \url{https://proceedings.mlsys.org/paper/2021/file/73278a4a86960eeb576a8fd4c9ec6997-Paper.pdf}.

\bibitem[Swersky et~al.(2013)Swersky, Snoek, and Adams]{swersky2013multitask}
Swersky, K., Snoek, J., and Adams, R.~P.
\newblock Multi-task bayesian optimization.
\newblock In Burges, C., Bottou, L., Welling, M., Ghahramani, Z., and
  Weinberger, K. (eds.), \emph{Advances in Neural Information Processing
  Systems}, volume~26. Curran Associates, Inc., 2013.
\newblock URL
  \url{https://proceedings.neurips.cc/paper/2013/file/f33ba15effa5c10e873bf3842afb46a6-Paper.pdf}.

\bibitem[Van~Rossum \& Drake(2009)Van~Rossum and Drake]{python}
Van~Rossum, G. and Drake, F.~L.
\newblock \emph{Python 3 Reference Manual}.
\newblock CreateSpace, Scotts Valley, CA, 2009.
\newblock ISBN 1441412697.

\bibitem[Vaswani et~al.(2017)Vaswani, Shazeer, Parmar, Uszkoreit, Jones, Gomez,
  Kaiser, and Polosukhin]{vaswani2017attentionisallyouneed}
Vaswani, A., Shazeer, N., Parmar, N., Uszkoreit, J., Jones, L., Gomez, A.~N.,
  Kaiser, {\L}., and Polosukhin, I.
\newblock Attention is all you need.
\newblock \emph{Advances in neural information processing systems}, 30, 2017.

\bibitem[Vo{\ss} et~al.(2012)Vo{\ss}, Martello, Osman, and
  Roucairol]{voss2012meta}
Vo{\ss}, S., Martello, S., Osman, I.~H., and Roucairol, C.
\newblock \emph{Meta-heuristics: Advances and trends in local search paradigms
  for optimization}.
\newblock Springer Science \& Business Media, 2012.

\bibitem[Wang et~al.(2020{\natexlab{a}})Wang, Wilder, Perrault, and
  Tambe]{wang2020surrogates}
Wang, K., Wilder, B., Perrault, A., and Tambe, M.
\newblock Automatically learning compact quality-aware surrogates for
  optimization problems.
\newblock \emph{Advances in Neural Information Processing Systems},
  33:\penalty0 9586--9596, 2020{\natexlab{a}}.

\bibitem[Wang et~al.(2020{\natexlab{b}})Wang, Fonseca, and
  Tian]{wang2020lamcts}
Wang, L., Fonseca, R., and Tian, Y.
\newblock Learning search space partition for black-box optimization using
  monte carlo tree search.
\newblock \emph{Advances in Neural Information Processing Systems},
  33:\penalty0 19511--19522, 2020{\natexlab{b}}.

\bibitem[Wang et~al.(2021{\natexlab{a}})Wang, Xie, Li, Fonseca, and
  Tian]{wang2021sample}
Wang, L., Xie, S., Li, T., Fonseca, R., and Tian, Y.
\newblock Sample-efficient neural architecture search by learning actions for
  monte carlo tree search.
\newblock \emph{IEEE Transactions on Pattern Analysis and Machine
  Intelligence}, 2021{\natexlab{a}}.

\bibitem[Wang et~al.(2019)Wang, Donti, Wilder, and Kolter]{wang2019satnet}
Wang, P.-W., Donti, P., Wilder, B., and Kolter, Z.
\newblock Satnet: Bridging deep learning and logical reasoning using a
  differentiable satisfiability solver.
\newblock In \emph{International Conference on Machine Learning}, pp.\
  6545--6554. PMLR, 2019.

\bibitem[Wang et~al.(2021{\natexlab{b}})Wang, Liu, Zhao, Liu, Liu, and
  Yan]{wang2021surrogate}
Wang, X., Liu, Y., Zhao, J., Liu, C., Liu, J., and Yan, J.
\newblock Surrogate model enabled deep reinforcement learning for hybrid energy
  community operation.
\newblock \emph{Applied Energy}, 289:\penalty0 116722, 2021{\natexlab{b}}.

\bibitem[Wetzstein et~al.(2020)Wetzstein, Ozcan, Gigan, Fan, Englund,
  Solja{\v{c}}i{\'c}, Denz, Miller, and Psaltis]{wetzstein2020photomlaccel}
Wetzstein, G., Ozcan, A., Gigan, S., Fan, S., Englund, D., Solja{\v{c}}i{\'c},
  M., Denz, C., Miller, D.~A., and Psaltis, D.
\newblock Inference in artificial intelligence with deep optics and photonics.
\newblock \emph{Nature}, 588\penalty0 (7836):\penalty0 39--47, 2020.

\bibitem[Wilder et~al.(2019{\natexlab{a}})Wilder, Dilkina, and
  Tambe]{wilder2019melding}
Wilder, B., Dilkina, B., and Tambe, M.
\newblock Melding the data-decisions pipeline: Decision-focused learning for
  combinatorial optimization.
\newblock In \emph{Proceedings of the AAAI Conference on Artificial
  Intelligence}, volume~33, pp.\  1658--1665, 2019{\natexlab{a}}.

\bibitem[Wilder et~al.(2019{\natexlab{b}})Wilder, Ewing, Dilkina, and
  Tambe]{wilder2019graphcombopt}
Wilder, B., Ewing, E., Dilkina, B., and Tambe, M.
\newblock End to end learning and optimization on graphs.
\newblock \emph{Advances in Neural Information Processing Systems}, 32,
  2019{\natexlab{b}}.

\bibitem[Wolsey(1982)]{wolsey1982analysis}
Wolsey, L.~A.
\newblock An analysis of the greedy algorithm for the submodular set covering
  problem.
\newblock \emph{Combinatorica}, 2\penalty0 (4):\penalty0 385--393, 1982.

\bibitem[Wolsey(2007)]{wolsey2007mixed}
Wolsey, L.~A.
\newblock Mixed integer programming.
\newblock \emph{Wiley Encyclopedia of Computer Science and Engineering}, pp.\
  1--10, 2007.

\bibitem[Yang et~al.(2021)Yang, Zhang, Cummins, Cui, Steiner, Wang, Gonzalez,
  Klein, and Tian]{yang2021learning}
Yang, K., Zhang, T., Cummins, C., Cui, B., Steiner, B., Wang, L., Gonzalez,
  J.~E., Klein, D., and Tian, Y.
\newblock Learning space partitions for path planning.
\newblock \emph{Advances in Neural Information Processing Systems},
  34:\penalty0 378--391, 2021.

\bibitem[Ye et~al.(2019)Ye, Zhang, and Sun]{ye2019automated}
Ye, Y., Zhang, X., and Sun, J.
\newblock Automated vehicle’s behavior decision making using deep
  reinforcement learning and high-fidelity simulation environment.
\newblock \emph{Transportation Research Part C: Emerging Technologies},
  107:\penalty0 155--170, 2019.

\bibitem[Zha et~al.(2022{\natexlab{a}})Zha, Feng, Bhushanam, Choudhary, Nie,
  Tian, Chae, Ma, Kejariwal, and Hu]{zha2022autoshard}
Zha, D., Feng, L., Bhushanam, B., Choudhary, D., Nie, J., Tian, Y., Chae, J.,
  Ma, Y., Kejariwal, A., and Hu, X.
\newblock Autoshard: Automated embedding table sharding for recommender
  systems.
\newblock In \emph{Proceedings of the 28th ACM SIGKDD Conference on Knowledge
  Discovery and Data Mining}, pp.\  4461--4471, 2022{\natexlab{a}}.

\bibitem[Zha et~al.(2022{\natexlab{b}})Zha, Feng, Tan, Liu, Lai, Bhargav, Tian,
  Kejariwal, and Hu]{zha2022dreamshard}
Zha, D., Feng, L., Tan, Q., Liu, Z., Lai, K.-H., Bhargav, B., Tian, Y.,
  Kejariwal, A., and Hu, X.
\newblock Dreamshard: Generalizable embedding table placement for recommender
  systems.
\newblock In \emph{Advances in Neural Information Processing Systems},
  2022{\natexlab{b}}.

\bibitem[Zha et~al.(2023)Zha, Feng, Luo, Bhushanam, Liu, Hu, Nie, Huang, Tian,
  Kejariwal, et~al.]{zha2023pre}
Zha, D., Feng, L., Luo, L., Bhushanam, B., Liu, Z., Hu, Y., Nie, J., Huang, Y.,
  Tian, Y., Kejariwal, A., et~al.
\newblock Pre-train and search: Efficient embedding table sharding with
  pre-trained neural cost models.
\newblock In \emph{Sixth Conference on Machine Learning and Systems}, 2023.

\bibitem[Zhang \& Dietterich(1995)Zhang and Dietterich]{zhang1995rljobshop}
Zhang, W. and Dietterich, T.~G.
\newblock A reinforcement learning approach to job-shop scheduling.
\newblock In \emph{IJCAI}, volume~95, pp.\  1114--1120. Citeseer, 1995.

\bibitem[Zhao et~al.(2022)Zhao, Wang, Yang, Zhang, Guo, and
  Tian]{zhao2022multiobjective}
Zhao, Y., Wang, L., Yang, K., Zhang, T., Guo, T., and Tian, Y.
\newblock Multi-objective optimization by learning space partition.
\newblock In \emph{International Conference on Learning Representations}, 2022.
\newblock URL \url{https://openreview.net/forum?id=FlwzVjfMryn}.

\bibitem[Zhou et~al.(2020)Zhou, Roy, Abdolrashidi, Wong, Ma, Xu, Liu,
  Phothilimtha, Wang, Goldie, et~al.]{zhou2020compiler}
Zhou, Y., Roy, S., Abdolrashidi, A., Wong, D., Ma, P., Xu, Q., Liu, H.,
  Phothilimtha, P., Wang, S., Goldie, A., et~al.
\newblock Transferable graph optimizers for ml compilers.
\newblock \emph{Advances in Neural Information Processing Systems},
  33:\penalty0 13844--13855, 2020.

\end{thebibliography}
\bibliographystyle{icml2023}
\newpage
\onecolumn
\appendix

\section{Proofs}

\suffcondition*
\begin{proof}
Since the dataset is a $\epsilon/L$-cover, for any $\vy\in Y$, there exists at least one $\vy_i$ so that $\|\vy-\vy_i\|_2 \le \epsilon/L$. Let $\vy_\nn$ be the nearest neighbor of $\vy$, and we have:
\begin{equation}
    \|\vy - \vy_{\nn}\|_2 \le \|\vy - \vy_{i}\|_2 \le \epsilon/L
\end{equation}
From the Lipschitz condition and the definition of 1-nearest-neighbor classifier ($\hat\vphi(\vy) = \vphi(\vy_\nn)$), we know that 
\begin{equation}
    \|\vphi(\vy) - \hat \vphi(\vy)\|_2 = \|\vphi(\vy) - \vphi(\vy_{\nn})\|_2 \le L\|\vy - \vy_{\nn}\|_2 \le \epsilon  
\end{equation}
\end{proof}

\lowerboundcomplexity*
\begin{proof}
We prove by contradiction. If $N < N_0(\epsilon)$, then for each training sample $(\vy_i, \vphi_i)$, we create a ball $B_i := B\left(\vy_i, \epsilon / L\right)$. Since 
\begin{equation}
\vol\left( \bigcup_{i=1}^N B_i \cap Y \right) \le \vol\left( \bigcup_{i=1}^N B_i \right) \le  \sum_{i=1}^N \vol(B_i) = N \vol_0 \left(\frac{\epsilon}{L}\right)^d < \vol(Y) 
\end{equation}
Therefore, there exists at least one $\vy\in Y$ so that $\vy \notin B_i$ for any $1\le i\le N$. This means that $\vy$ is not $\epsilon/L$-covered. 
\end{proof}

\infinitelip*
\begin{proof}
Let $R_1, R_2, \ldots, R_K$ be the $K = \kappa(\vphi(Y))$ connected components of $\vphi(Y)$, and $Y_1, Y_2, \ldots, Y_J$ be the $J = \kappa(Y)$ connected components of $Y$. From the condition, we know that $\min_{k\neq k'} \mathrm{dist}(R_k, R_{k'}) = d_{\min} > 0$. 

We have $R_k \cap R_{k'} = \emptyset$ for $k\neq k'$. Each $R_k$ has a pre-image $S_k := \vphi^{-1}(R_k) \subseteq Y$. These pre-images $\{S_k\}_{k=1}^K$ form a partition of $Y$ since
\begin{itemize}
    \item $S_k \cap S_{k'} = \emptyset$ for $k\neq k'$ since any $\vy\in Y$ cannot be mapped to more than one connected components; 
    \item $\bigcup_{k=1}^K S_k = \bigcup_{k=1}^K \vphi^{-1}(R_{k}) = \vphi^{-1}\left(\bigcup_{k=1}^K R_k\right) = \vphi^{-1}(\vphi(S)) = S$. 
\end{itemize}
Since $K = \kappa(\vphi(Y)) > \kappa(Y)$, by pigeonhole principle, there exists one $Y_j$ that contains at least part of the two pre-images $S_k$ and $S_{k'}$ with $k \neq k'$. This means that
\begin{equation}
    S_k \cap Y_j \neq \emptyset,\quad S_{k'} \cap Y_j \neq \emptyset
\end{equation}
Then we pick $\vy \in S_k \cap Y_j$ and $\vy' \in S_{k'} \cap Y_j$. Since $\vy,\vy'\in Y_j$ and $Y_j$ is a connected component, there exists a continuous path $\gamma: [0,1]\mapsto Y_j$ so that $\gamma(0) = \vy$ and $\gamma(1) = \vy'$. Therefore, we have $\vphi(\gamma(0)) \in R_k$ and $\vphi(\gamma(1)) \in R_{k'}$. Let $t_0 := \sup\{t: t\in [0,1], \vphi(\gamma(t)) \in R_k\}$, then $0 \le t_0 < 1$. For any sufficiently small $\epsilon > 0$, we have:
\begin{itemize}
    \item By the definition of $\sup$, we know there exists $t_0 - \epsilon \le t' \le t_0$ so that $\vphi(\gamma(t')) \in R_k$. 
    \item Picking $t'' = t_0 + \epsilon < 1$, then $\vphi(\gamma(t'')) \in R_{k''}$ with some $k''\neq k$. 
\end{itemize}
On the other hand, by continuity of the curve $\gamma$, there exists a constant $C(t_0)$ so that $\|\gamma(t') - \gamma(t'')\|_2 \le C(t_0)\|t'-t''\|_2 \le 2C(t_0)\epsilon$. Then we have 
\begin{equation}
    L = \max_{\vy,\vy'\in Y} \frac{\|\vphi(\vy)-\vphi(\vy')\|_2}{\|\vy-\vy'\|_2} \ge \frac{\|\vphi(\gamma(t')) - \vphi(\gamma(t''))\|_2}{\|\gamma(t') - \gamma(t'')\|_2} \ge \frac{d_{\min}}{2C(t_0)\epsilon} \rightarrow +\infty
\end{equation}

\end{proof}

\section{Experiment Details}
\label{sec:experiment_appendix}

\subsection{Setups}
Experiments are performed on a cluster of identical machines, each with 4 Nvidia A100 GPUs and 32 CPU cores, with 1T of RAM and 40GB of GPU memory. Additionally, we perform all operations in Python \citep{python} using Pytorch \citep{pytorch}. For embedding table placement, the nonlinear cost estimator is trained for 200 iterations and the offline-trained models of Dreamshard and \ours{}-\trainvershort{} are trained against the pretrained cost estimator for 200 iterations. The DLRM Dataset \cite{DLRM19} is available at \url{https://github.com/facebookresearch/dlrm_datasets}, and the dreamshard \citep{zha2022dreamshard} code is available at \url{https://github.com/daochenzha/dreamshard}. Additional details on dreamshard's model architecture and features can be obtained in the paper and codebase. {Training time for the networks used in \ours{}-\trainver{} and \ours{}-\hybridver{} are on average 8 hours for the inverse photonic design settings and 6, 21, 39, 44, 50, 63 minutes for DLRM 10, 20, 30, 40, 50, 60 settings respectively}.

\subsection{Network Architectures}
\subsubsection{Embedding Table Sharding}
The table features are the same used in \cite{zha2022dreamshard}, and sinusoidal positional encoding~\cite{vaswani2017attentionisallyouneed} is used as device features so that the learning model is able to break symmetries between the different tables and effectively group them onto homogeneous devices. The table and device features are concatenated and then fed into Dreamshard's initial fully-connected table encoding module to obtain scalar predictions $\surrobjlin{}_{t,d}$ for each desired objective coefficient. {The architecture is trained with the Adam optimizer with learning rate 0.0005}. Here, we use the dreamshard backbone to predict coefficients for each table-device pair. We add more output dimensions to the dreamshard backbone, ensuring that we output the desired number of coefficients.

\subsubsection{Inverse Photonic Design}
\textbf{Network architectures}. The input design specification (a 2D image) is passed through a 3 layer convolutional neural network with ReLU activations and a final layer composed of filtering with the known brush shape. Then a tanh activation is used to obtain surrogate coefficients $\hat\vc$, one component for each binary input variable. {The architecture is trained with the Adam optimizer with learning rate 0.001}.

This is motivated by previous work \citep{schubert_inverse_2022} that also uses the fixed brush shape filter and tanh operation to transform the latent parameters into a continuous solution that is projected onto the space of physically feasible solutions. 

{
In each setting, optimization is done on a binary grid of different sizes to meet fabrication constraints, namely that a 3 by 3 cross must fit inside each fixed and void location. In the beam splitter the design is an $80\times 60$ grid, in mode converter it is a $40\times40$ grid, in waveguide bend it is a $40\times40$ grid, in wavelength division multiplexer it is an $80\times80$ grid.
}

\begin{table}
    \centering
    \begin{tabular}{l|l}
        Task  & Randomization  \\
        \hline\hline
        mode converter & randomize the right and left waveguide width \\
        bend setting   & randomize the waveguide width and length \\
        beam splitter  & randomize the waveguide separation, width and length \\
        wavelength division multiplexer & randomize the input and output waveguide locations
    \end{tabular}
    \caption{Task randomization of 4 different tasks in inverse photonic design.}
    \label{tab:my_label}
\end{table}

Previous work formulated the projection as finding a discrete solution that minimized the dot product of the input continuous solution and proposed discrete solution. The authors then updated the continuous solution by computing gradients of the loss with respect to the discrete solution and using pass-through gradients to update the continuous solution. By comparison, our approach treats the projection as an optimization problem and updates the objective coefficients so that the resulting projected solution moves in the direction of the desired gradient. 

To compute the gradient of this blackbox projection solver, we leverage the approach suggested by \cite{poganvcic2019diffbb} which calls the solver twice, once with the original coefficients, and again with coefficients that are perturbed in the direction of the incoming solution gradient as being an ``improved solution''. The gradient with respect to the input coefficients are then the difference between the ``improved solution'' and the solution for the current objective coefficients.




\section{Pseudocode}
{
Here is the pseudocode for the different variants of our algorithm. Each of these leverage a differentiable optimization solver to differentiate through the surrogate optimization problem.}

\begin{algorithm}[H]
   \caption{\ours{}-\optver{}}
   \label{alg:surco-zero}
\begin{algorithmic}
\STATE {\bfseries Input:} feasible region $\feas{}$, data $\vy$, objective $f$
\STATE $\vc \gets$ init\_surrogate\_coefs$(\vy)$
\WHILE{not converged}
    \STATE $\vx \gets \arg\min_{\vx\in \Omega(\vy)} \vc^\top \vx$
    \STATE loss $\gets f(\vx;\vy)$
    \STATE $\vc \gets $grad\_update$(\vc, \nabla_{\vc} \text{loss})$
\ENDWHILE
\STATE Return $\vx$
\end{algorithmic}
\end{algorithm}

\begin{algorithm}[H]
\caption{\ours{}-\trainver{} Training}\label{alg:surco-prior}
\begin{algorithmic}
\STATE {\bfseries Input:} feasible region $\feas{}$, data $\cD_\train = \{\vy_i\}_{i=1}^N$, objective $f$
\STATE $\theta \gets$ init\_surrogate\_model$()$
\WHILE {not converged}
    \STATE Sample batch $B=\{\vy_i\}_i^k \sim \cD_\train$
    \FOR{$\vy \in B$}
    \STATE $\hat{\vc} \gets \hat{\vc}(\vy;\theta)$
    \STATE $\vx \gets \arg\min_{\vx\in \Omega(\vy)} \vc^\top \vx$
    \STATE loss $\mathrel{+}= f(\vx;\vy)$
    \ENDFOR
    \STATE $\theta \gets $grad\_update$(\theta, \nabla_{\theta} \text{loss})$
\ENDWHILE
\STATE Return $\theta$
\end{algorithmic}
\end{algorithm}

\begin{algorithm}[H]
\caption{\ours{}-\trainver{} Deployment}\label{alg:surco-prior-test}
\begin{algorithmic}[1]
\STATE {\bfseries Input:} feasible region $\feas{}$, data $\cD_\train = \{\vy_i\}_{i=1}^N$, objective $f$, test instance $\vy_{\text{test}}$
\STATE $\theta \gets $ train \ours{}-\trainver{}$(\feas{}, \cD_\train, f)$
\STATE $\vc \gets \hat{\vc}(\vy;\theta)$
\STATE $\vx \gets \arg\min_{\vx\in \Omega(\vy)} \vc^\top \vx$
\STATE Return $\vx$
\end{algorithmic}
\end{algorithm}

\begin{algorithm}[H]
\caption{\ours{}-\hybridver{}}\label{alg:surco-hybrid}
\begin{algorithmic}[1]
\STATE {\bfseries Input:} feasible region $\feas{}$, data $\cD_\train = \{\vy_i\}_{i=1}^N$, objective $f$, test instance $\vy_{\text{test}}$
\STATE $\theta \gets $ train \ours{}-\trainver{}$(\feas{}, \cD_\train, f)$
\STATE $\vc \gets \hat{\vc}(\vy;\theta)$
\WHILE{not converged}
    \STATE $\vx \gets \arg\min_{\vx\in \Omega(\vy)} \vc^\top \vx$
    \STATE loss $\gets f(\vx;\vy)$
    \STATE $\vc \gets $grad\_update$(\vc, \nabla_{\vc} \text{loss})$
\ENDWHILE
\STATE Return $\vx$
\end{algorithmic}
\end{algorithm}

\section{Additional Failed Baselines}
\paragraph{SOGA - Single Objective Genetic Algorithm}
Using PyGAD \citep{gad2021pygad}, we attempted several approaches for both table sharding and inverse photonics settings. While we were able to obtain feasible table sharding solutions, they underperformed the greedy baseline by $20\%$. Additionally, they were unable to find physically feasible inverse photonics solutions. We varied between random, swap, inversion, and scramble mutations and used all parent selection methods but were unable to find viable solutions.

\paragraph{DFL - A Derivative-Free Library}
We could not easily integrate DFLGEN \citep{liuzzi2015dfl} into our pipelines since it operates in fortran and we needed to specify the feasible region with python in the ceviche challenges. DFLINT works in python but took more than 24 hours to run on individual instances which reached a timeout limit. We found that the much longer runtime made this inapplicable for the domains of interest. 



\paragraph{Nevergrad} We enforced integrality in Nevergrad \citep{nevergrad} using choice variables which selected between 0 and 1. This approach was unable to find feasible solutions for inverse photonics in less than 10 hours. For table sharding we obtained solutions by using a choice variable for each table, selecting one of the available devices. This approach was not able to outperform the greedy baseline and took longer time so it was strictly dominated by the greedy approach.

\paragraph{Solution Prediction} We made several attempts at training solution predictors for each of our domains. We label each problem instance with the best-known solution obtained (including those obtained via SurCo). Note that predicting feasible solutions to combinatorial optimization problems is nontrivial for general settings.

We evaluate solution prediction architectures in each setting. The models here match the architecture of SurCo-prior but the output is fed through a sigmoid transformation to get predictions in [0,1]. In nonlinear shortest path we use a GCN architecture and predict [0,1] whether edges are in the shortest s-t path. Not surprisingly, we found that predicting solutions to combinatorial problems is a nontrivial problem, further motivating the use of SurCo which ensures combinatorial feasibility of the generated solution.

Note that the solutions predicted by the networks may not be binary (and thus not feasible). We then round the individual decision variables to get binary predictions. Empirically, we found that our predictions are very close to binary, indicating that rounding is more a numerical exactness operation than an algorithmic decision, with the largest distance from any original to rounded value being 0.0008 for inverse photonics, 0.0001 for nonlinear shortest path, and 0.0007 for the assignment problem of table sharding.

We evaluate the results on unseen test instances in Table \ref{tab:solution_pred} and find that these solution prediction approaches don’t yield combinatorially feasible solutions. We present machin learning performance in the table below to verify that the predictive models perform ``well'' in terms of standard machine learning evaluation even though they fail to generate feasible solutions.

\begin{table*}[ht]
\scriptsize
\centering
\begin{tabular}{r|ccc}
\toprule
Setting   & Decision Variable Accuracy Average & Solution Accuracy & Solution Feasibility Rate \\
\midrule
Inverse Photonics - Sigmoid            & 87\%                                                            & 0\%                                            & 0\%                                                    \\
Nonlinear Shortest Path - Sigmoid      & 95\%                                                            & 0\%                                            & 0\%                                                    \\
Table Sharding - Sigmoid               & 92\%                                                            & 0\%                                            & 0\%                                                    \\
Table Sharding - Softmax               & 88\%                                                            & 0\%                                            & 0\%                                                    \\
Table Sharding - Softmax +   Iterative & 70\%                                                            & 0\%                                            & 100\%                                                 \\
\bottomrule
\end{tabular}
\caption{Solution prediction results, most methods give infeasible solutions.}\label{tab:solution_pred}
\end{table*}

\begin{table*}[ht]
\scriptsize
\centering
\begin{tabular}{r|c}
\toprule
Setting & \% Latency Increase vs Domain Heuristic (worst baseline) \\
\midrule
DLRM-10 & 6\%                                                      \\
DLRM-20 & 5\%                                                      \\
DLRM-30 & 9\%                                                      \\
DLRM-40 & 7\%                                                      \\
DLRM-50 & 3\%                                                      \\
DLRM-60 & 11\%                                                    \\
\bottomrule
\end{tabular}
\caption{Comparison of only feasible solution prediction method against worst baseline.}\label{tab:solution_pred_table_shard}
\end{table*}

We also iterate on table sharding to produce two more domain-specific approaches. We evaluate a model variant which assigns each table into one of the 4 devices using softmax, which empirically fails to yield feasible solutions that meet device memory limits for any of our instances. We further develop a method called Softmax + Iterative which iteratively assigns the most likely table-device assignment as long as the device has enough memory to hold the device. Luckily, this Softmax + Iterative method empirically yields feasible solutions in this setting but we note that this approach is not guaranteed to terminate in feasible solutions, unlike SurCo. To see why Softmax + Iterative does not necessarily guarantee feasible termination, consider assigning 3 tables (2 small and 1 large) to 2 devices each with memory limit of 2, the small tables have memory 1 and the large table has memory 2. If the model’s highest assignment probability is on the small tables being evenly distributed across devices, the algorithm will first assign the small tables to devices 1 and 2 but stall because it is unable to assign the large table since neither device has enough remaining capacity. We present results for this Softmax + Iterative approach compared to our domain heuristic which is the worst performing baseline in Table \ref{tab:solution_pred_table_shard}.

For each setting, we evaluate the three metrics:
\begin{itemize}
    \item \textbf{Decision Variable Accuracy Average}, is the average percent of variables which are correctly predicted.
    \item \textbf{The solution accuracy}, is the rate of predicting the full solution correctly (all decision variables predicted correctly).
    \item \textbf{The solution feasibility rate}, is the percent of instances for which the predicted solution satisfies the constraints.

\end{itemize}

\end{document}